\pdfoutput=1
\documentclass[11pt,a4paper]{article}

\usepackage[T1]{fontenc}
\usepackage[utf8]{inputenc}
\usepackage{lmodern}
\usepackage{microtype}
\usepackage[margin=1in]{geometry}
\usepackage{amsmath,amssymb,amsthm}
\usepackage{graphicx}
\usepackage{booktabs}
\usepackage[table]{xcolor}
\usepackage{threeparttable}
\usepackage{algorithm}
\usepackage{algpseudocode}
\usepackage{subcaption}
\usepackage[round,authoryear]{natbib}
\usepackage{url}
\usepackage[hidelinks]{hyperref}
\usepackage{authblk}
\DeclareMathOperator{\clip}{clip}
\DeclareMathOperator*{\argmax}{arg\,max}
\DeclareMathOperator*{\argmin}{arg\,min}
\newtheorem{theorem}{Theorem}
\newtheorem{proposition}{Proposition}
\newtheorem{lemma}{Lemma}
\newtheorem{remark}{Remark}

\newcommand{\Sset}{\mathcal S}
\newcommand{\Aset}{\mathcal A}
\newcommand{\simplex}{\Delta(\Aset)}
\newcommand{\E}{\mathbb E}
\newcommand{\KL}{\operatorname{KL}}
\newcommand{\1}{\mathbf{1}}
\newcommand{\ip}[2]{\left\langle #1,#2\right\rangle}
\newcommand{\norm}[1]{\left\lVert #1\right\rVert}

\title{Towards Better Training Signal: Advantage Clipped Policy Optimization}

\author[1]{Ruichuan Huang}
\author[2]{Jinghan Liu}
\author[2]{Congliang Chen}

\affil[1]{MIT}
\affil[2]{Shenzhen Loop Area Institute}

\affil[ ]{\texttt{
ruichuan@mit.edu,
jinghanliu@slai.edu.cn,
chencongliang@slai.edu.cn
}}

\date{}

\begin{document}

\maketitle

\begin{abstract}

Reinforcement learning (RL) has become a cornerstone for improving the reasoning capabilities of large language models (LLMs), but the need for on-policy data substantially limits training efficiency. Reusing off-policy data through importance sampling (IS) can improve efficiency but introduce considerable instability. Hence, algorithms such as PPO and GRPO widely adopt IS-ratio clipping to stabilize training. 
However, training stability and gradient estimate are mainly determined by the product of IS ratio and advantage. To further stabilize training, we propose ACPO, which clips the product of the IS ratio and the advantage, leading to more stable gradient estimates. We also establish a connection between ACPO and gradient clipping in policy mirror descent (PMD), which is a standard technique to stabilize optimization process,  and prove the convergence of clipped-PMD under the standard RL setting. Experiments on widely used mathematical reasoning benchmarks show that ACPO consistently outperforms PPO and GRPO in both accuracy and training efficiency, delivering 4-6\% points gains on standard math benchmarks, with Qwen3-8B+PPO. Hence, ACPO is a practical and effective alternative to conventional IS-ratio clipping for RL post-training of LLMs.

\end{abstract}

\section{Introduction}
\label{sec:introduction}
Large language models (LLMs) now serve as a core part for modern natural language understanding and generation, driven by rapid growth in model scale and training data \citep{brown2020LLM,long2022RLHF}. However, strong performance on challenging tasks requires more than retrieval and reproducing knowledge; models must also decompose complex problems, integrate intermediate conclusions, generalize across contexts, and generate logically coherent solutions. Reinforcement learning with objectively assessable feedback has recently shown considerable potential for cultivating such abilities, particularly  with verifiable, outcome-based rewards \citep{Guo2025r1,lambert2025tulu3pushingfrontiers}. Under this framework, the training objective is to maximize the expected reward of model outputs under a given prompt distribution. The primary challenge lies in achieving efficient and stable training rather than formulating the objective. The need for on-policy data substantially limits training efficiency: on-policy trajectories require generation after each policy update, which is particularly costly for hard reasoning tasks. Reusing off-policy data through importance sampling (IS) can reduce this cost, but the resulting importance weights lead to unstable gradient estimate.

A central approach to stabilizing reinforcement learning for LLM post-training is the actor–critic framework, in which the critic estimates a baseline for the actor’s policy-gradient updates, reducing the variance of gradient estimation \citep{barto1983actor-critic}. And to address the training instability caused by the partially off-policy trajectories, PPO introduces a clipping mechanism on the IS ratio, thereby constraining excessive policy updates and improving training stability \citep{schulman2017ppo}.

Building on the clipping mechanism introduced in PPO, numerous studies have proposed alternative clipping strategies centered on the importance-sampling (IS) ratio \citep{minimax2025CISPO,yue2025vapo,yu2025dapo,ye2020dualclip,liu2026ABC}. Despite this growing body of work, our understanding of ratio clipping remains limited, even for the standard clipping rule used in PPO. However, a principled understanding of when and why this mechanism improves training stability is still lacking. Although the IS ratio measures the deviation of the current policy from the rollout policy, policy update and gradient estimate are mainly determined by the product of IS ratio and advantage, where the advantage determines both the direction and the magnitude of the update. 

Our work starts by directly looking at the gradient estimator. Then we try to control the updating magnitude by controlling the product of the importance-sampling (IS) ratio and the advantage, which determines the magnitude and direction of each sample's policy-gradient contribution.  We propose \textbf{Advantage Clipped Policy Optimization} (ACPO), which clips the product of IS ratio and advantage to $[-\alpha,\alpha]$. Unlike conventional PPO-style methods that clip only the IS ratio, ACPO directly controls the scalar coefficient of each policy-gradient contribution and can be readily integrated into both PPO and GRPO. Under explicit assumptions, we show that ACPO has no larger gradient second moment than PPO and provides a variance-reduction guarantee under an additional assumption.

Another starting point of our work is the policy mirror descent in optimization theory. Clipping is widely used over gradient in optimization, and we want ask is there any connection between gradient clipping and clipping in policy optimization. Our work connects ACPO clipping mechanism to gradient clipping in policy mirror descent and establishes convergence guarantees for the corresponding gradient-clipped PMD in standard RL setting. It is a novel connection between these clipping designs, where gradient clipping has been widely use in both theory and practical  optimizer to stabilize optimization process. Our result shows the ACPO clipping principle derived from gradient-clipped PMD also stabilizes the policy optimization process.

\subsection{Our Contribution}

\textbf{Algorithm. } We propose Advantage Clipped Policy Optimization (ACPO), which applies clipping to the product of the importance-sampling ratio and advantage. The resulting objective retains sample gradient contributions within the clipping range and suppresses those outside it. \\
\textbf{Theory.} We characterize the second moment of the ACPO gradient estimator and establish a lower second moment than PPO under mild assumptions. Under additional assumption, this comparison yields a variance-reduction guarantee. We further establish a connection to policy mirror descent with gradient clipping, deriving convergence bounds for the associated PMD algorithm in standard RL settings. Together, these results provide an optimization perspective on ACPO.\\
\textbf{Experiments.} We evaluate ACPO on four mathematical reasoning benchmarks using Qwen2.5-Math-7B and Qwen3 models(1.7B, 4B, and 8B). ACPO improves the best observed weighted accuracy in all eight matched comparisons, yielding gains of 1.6–4.0 percentage points over PPO and 0.6–1.5 percentage points over GRPO, while achieving each baseline’s best observed performance in substantially fewer training steps (speedup 6.25x over PPO and 1.7x over GRPO). Moreover, PPO-AC achieves competitive performance relative to GRPO while using only one rollout per prompt.

\section{Preliminaries}
\label{sec:preliminaries}
We first introduce some basic notation of Markov decision process(MDP), which we will use in following paper. We consider a $\gamma$-discounted MDP $\mathcal{M}=(\mathcal{S},\mathcal{A},P,r,\gamma,\rho)$, where $\mathcal{S}$ and $\mathcal{A}$ are the state and action spaces, $P(\cdot\mid s,a)$ is the transition kernel, $r(s,a)$ denotes the reward function (or $c(s,a)$ the cost function when minimizing cost), $\gamma\in[0,1)$ is the discount factor, and $\rho$ is the initial-state distribution. A stochastic policy $\pi$ specifies an action distribution $\pi(\cdot\mid s)$ at each state, and induces trajectories according to $s_0\sim\rho$, $a_t\sim\pi(\cdot\mid s_t)$, and $s_{t+1}\sim P(\cdot\mid s_t,a_t)$. The performance of $\pi$ is measured by the expected discounted return $J_{\rho}(\pi)=\mathbb{E}_{\rho,\pi}[\sum_{t=0}^{\infty}\gamma^t r(s_t,a_t)]$. We denote by $d_{\rho}^{\pi}(s)=(1-\gamma)\sum_{t=0}^{\infty}\gamma^t\Pr_{\rho,\pi}(s_t=s)$ the normalized discounted frequency of visiting state $s$, and by $\mu_{\rho}^{\pi}(s,a)=d_{\rho}^{\pi}(s)\pi(a\mid s)$ the corresponding discounted state-action visitation distribution. The state-value and action-value functions of $\pi$ are defined as $V^{\pi}(s)=\mathbb{E}_{\pi}[\sum_{t=0}^{\infty}\gamma^t r(s_t,a_t)\mid s_0=s]$ and $Q^{\pi}(s,a)=\mathbb{E}_{\pi}[\sum_{t=0}^{\infty}\gamma^t r(s_t,a_t)\mid s_0=s,a_0=a]$, respectively. Then,  the advantage function is given by the difference of $Q^\pi(s,a)$ and $V^\pi(s)$: $A^{\pi}(s,a)=Q^{\pi}(s,a)-V^{\pi}(s)$, which measures the relative benefit of choosing action $a$ instead of following the policy's average action at state $s$.

Now we can consider more general $Q$ function that
$Q^{\pi}(s,a)=\mathbb{E}_{\pi}[\sum_{t=0}^{\infty}\gamma^t [c(s_t,a_t)+h^\pi(s_t)]\mid s_0=s,a_0=a]$ in minimizing cost case. Here $h^\pi$ is a closed convex function w.r.t.\ the policy $\pi$, i.e.,
there exists  $\mu \ge 0$ satisfying
\begin{equation}
h^\pi(s)
-
\left[
h^{\pi'}(s)
+
\left\langle
(\partial h)^{\pi'}(s,\cdot),
\pi(\cdot \mid s)-\pi'(\cdot \mid s)
\right\rangle
\right]
\ge
\mu D_{\pi'}^\pi(s),
\end{equation}
where $\langle\cdot,\cdot\rangle$ is the inner product over $\mathcal A$, $(\partial h)^{\pi'}(s,\cdot)$ is a subgradient of $\pi\mapsto h^\pi(s)$ at $\pi'$, and $D_{\pi'}^\pi(s)$ is a Bregman divergence. The Bregman
divergence induced by a differentiable and strongly convex mirror map $\psi$ is defined as
$
D_{\psi}(x,y)
=
\psi(x)-\psi(y)
-\left\langle \nabla\psi(y),x-y\right\rangle 
$. Setting $h^\pi=0$ recovers the standard action-value function; setting $h^\pi(s)=\mu D_{\pi_0}^\pi(s)$ with $\mu>0$ and $D$ equals to $\KL$, then $Q^\pi$ reduces to the so-called entropy
regularized action-value function. 

\subsection{Reinforcement Learning for LLMs}
In this section, we briefly review the most popular reinforcement learning methods for large language models. We formulate the generation process of a large language model as a MDP. Let $x \sim \mathcal{D}$ denote an input prompt and
$y=(y_1,\ldots,y_T)$ denote the tokens generated by the model. At generation
step $t$, the state is defined as the prompt together with the previously
generated tokens $s_t=(x,y_{<t})$,
and the action $a_t=y_t$ corresponds to selecting the next token. The policy
$\pi_\theta(a_t\mid s_t)$ is parameterized by the policy model, and the
transition is deterministic: $s_{t+1}=(x,y_{\leq t})$.
An episode terminates when an end-of-sequence token is generated or when the maximum generation length is reached.
\paragraph{Proximal Policy Optimization (PPO).} PPO~\citep{schulman2017ppo} stabilizes policy optimization by introducing a clipped surrogate
objective. Let $\pi_{\theta_{\mathrm{old}}}$ denote
the policy of training trajectories, and define the
importance sampling ratio as
$
r_t(\theta)
=
\frac{\pi_\theta(a_t\mid s_t)}
     {\pi_{\theta_{\mathrm{old}}}(a_t\mid s_t)}.
$
The clipped PPO objective is
{ \begin{align*}
\mathcal{J}^{\mathrm{PPO}}(\theta)
=
\mathbb{E}_{(s_t,a_t)\sim\pi_{\theta_{\mathrm{old}}}}
\left[
\min\left(
r_t(\theta)\widehat{A}_t,\,
\operatorname{clip}
\bigl(r_t(\theta),1-\epsilon,1+\epsilon\bigr)
\widehat{A}_t
\right)
 \right],    
\end{align*}}

where $\epsilon$ is the clipping parameter and $\widehat{A}_t$ is an
estimate of the advantage function. In practice, the advantage is often
computed using generalized advantage estimation  \citep[GAE;][]{schulman2017ppo}: 
$
\widehat{A}_t
=
\sum_{l=0}^{T-t-1}
(\lambda)^l\delta_{t+l}$, and $
\delta_t
=
r_t+\gamma V_\phi(s_{t+1})-V_\phi(s_t)$,
where $V_\phi$ is a learned value function. In policy optimization,
PPO first introduce the clipping technique and achieve empirical success.
\paragraph{Group Relative Policy Optimization (GRPO).}
GRPO \citep{shao2024grpo} is a critic-free method that, in practical
implementations for LLMs like DeepSeek-R1\citep{Guo2025r1}, samples \(G\) responses
\(o^{(1)}, \ldots, o^{(G)}\) for each prompt \(x\) and computes advantages
by normalizing rewards within each prompt group. In the MDP notation
above, this corresponds to: $o^{(i)}
=
\left(
a_1^{(i)},
a_2^{(i)},
\ldots,
a_{\lvert o^{(i)} \rvert}^{(i)}
\right)$ and $s_t^{(i)}
=
\left(
x,
a_{<t}^{(i)}
\right)$, where $x\sim \mathcal{D}$. Then, the advantage for the \(i\)-th response \(o^{(i)}\)
 is computed as:
$
\widehat{A}^{(i)}
=
\frac{
r\left(x,o^{(i)}\right)
-
\operatorname{mean}
\left(
\left\{
r\left(x,o^{(1)}\right),
\ldots,
r\left(x,o^{(G)}\right)
\right\}
\right)
}{
\operatorname{std}
\left(
\left\{
r\left(x,o^{(1)}\right),
\ldots,
r\left(x,o^{(G)}\right)
\right\}
\right)
}$.

 This response-level advantage
\(\widehat{A}^{(i)}\) is then used to replace the step-wise advantage
function \(\widehat{A}_h(s_h,a_h)\) in the PPO objective
\(\mathcal{J}^{\mathrm{PPO}}\). The GRPO objective also includes a KL-regularization term:
{\[
\begin{aligned}
\mathcal{J}^{\mathrm{GRPO}}(\theta)
={}&
\mathbb{E}_{
x\sim\mathcal{D},\,
\{o^{(i)}\}_{i=1}^{G}
\sim
\pi_{\theta_{\mathrm{old}}}(\cdot\mid x)
}
\Bigg[
\frac{1}{G}
\sum_{i=1}^{G}
\frac{1}{\lvert o^{(i)}\rvert}
\sum_{t=1}^{\lvert o^{(i)}\rvert}
\Bigg\{
\\
&\quad
\min
\Bigg[
r_t^{(i)}(\theta)
\widehat{A}^{(i)},
\operatorname{clip}
\left(
r_t^{(i)},
1-\varepsilon,
1+\varepsilon
\right)
\widehat{A}^{(i)}\Bigg]-
\beta
D_{\mathrm{KL}}
\left(
\pi \,\Vert\, \pi_{\mathrm{ref}}
\right)
\Bigg\}
\Bigg],
\end{aligned}
\]}
where $r_t^{(i)}(\theta)=\frac{
\pi_{\theta}
\left(a_t^{(i)}\mid s_t^{(i)}\right)
}{
\pi_{\theta_{\mathrm{old}}}
\left(a_t^{(i)}\mid s_t^{(i)}\right)
}$ and GRPO uses the same IS ratio clipping design as PPO.

\subsection{Mirror Descent and Policy Mirror Descent} \label{sec: mirror}

\paragraph{Mirror Descent.}
Mirror descent (MD) is a first-order optimization method for constrained convex optimization that generalizes gradient
descent to non-Euclidean geometries \citep{beck2003mirrordescent}. Consider the constrained optimization
problem $\min_{x\in\mathcal{X}} f(x)$,
where $\mathcal{X}$ is a convex set and $f:\mathcal{X}\rightarrow\mathbb{R}$
is a differentiable convex function. Given the current iterate $x_k$ and a gradient estimate $g_k$, the mirror
descent update with step size $\eta_k>0$ is
\[
x_{k+1}
=
\arg\min_{x\in\mathcal{X}}
\left\{
\left\langle g_k,x\right\rangle
+
\frac{1}{\eta_k}
D_{\psi}(x,x_k)
\right\}.
\]

For optimization
over the probability simplex
$
\Delta(\mathcal{A})
=
\left\{
p\in\mathbb{R}^{|\mathcal{A}|}_{+}:
\sum_{a\in\mathcal{A}}p^a=1
\right\}$,
a common choice is the negative entropy
$\psi(p)=\sum_{a\in\mathcal{A}}p^a\log p^a$.
Then, associated Bregman divergence is the Kullback--Leibler divergence,
$
D_{\psi}(p,q)
=
D_{\mathrm{KL}}(p\|q)
=
\sum_{a\in\mathcal{A}}
p^a\log\frac{p^a}{q^a}$. Under this geometry, mirror descent yields a
exponentiated gradient descent update:
 \[
x_{k+1}^a
=
\frac{
x_{k}^a\exp(-\eta_k g_{k}^a)
}{
\sum_{a'\in\mathcal{A}}
x_{k}^{a'}\exp(-\eta_k g_{k}^{a'})
},
\]
where $x_{k+1}^a$ and $g_k^a$ denote the $a$-th coordinate of $x_{k+1}$ and $g_k$. This update naturally preserves the simplex constraint and is therefore
particularly suitable for optimizing probability distributions.

\paragraph{Policy Mirror Descent.}
Policy mirror descent (PMD) extends mirror descent to policy optimization in
MDP. The main objective in RL to find an optimal policy minimizing the total cost with general $Q$ function can be formulate as an optimization problem with any $\rho$ satisfying $\rho(s)>0, \forall s$ and $\sum_{s} \rho(s)=1$:
\begin{align*}
\min_\pi f(\pi)&:=\mathbb E_{s\sim\rho}[V^\pi(s)]\quad \mathrm{s.t. }~ \pi(.\mid s) \in \Delta(\mathcal{A}), \forall s \in \mathcal{S}.
\end{align*}

Then applying MD to this optimization problem gives us the update (see \cite{pmdlan}):
\begin{align} \label{alg: PMD}
    \pi_{k+1}(\cdot\mid s)
=
\argmin_{p\in\Delta(\mathcal{A})}
\left\{
\left\langle Q^{\pi_k}(s,\cdot),p+h^p(s)\right\rangle
+
\frac{1}{\eta_k}
D_{\psi}\left(p,\pi_k(\cdot\mid s)\right)
\right\},
\end{align}
where $\eta_k>0$ is the step size, $h^\pi$ is a closed convex function w.r.t $p$ and $D_{\psi}$ is the Bregman divergence
induced by a mirror map $\psi$. For convenience, we also denote $D_{\psi}(p,\pi_k(\cdot\mid s))$ as $D^{p}_{\pi_k}(s)$. The first term encourages the policy to assign
more probability to actions with high estimated values, while the divergence
term constrains the new policy to remain close to the current policy. The
action-value function can equivalently be replaced by the advantage function
$A^{\pi_k}$, since subtracting a state-dependent baseline does not change the
solution.

When the negative entropy is used as the mirror map and $h^p=0$, the update has the closed-form expression
$
\pi_{k+1}(a\mid s)
=
\frac{
\pi_k(a\mid s)
\exp\left(\eta_k A^{\pi_k}(s,a)\right)
}{
\sum_{a'\in\mathcal{A}}
\pi_k(a'\mid s)
\exp\left(\eta_k A^{\pi_k}(s,a')\right)
}$.
Thus, PMD exponentially increases the probabilities of actions with positive
advantages and decreases those with negative advantages.

\section{Advantage Clipped Policy Optimization}
TRPO \cite{schulman2015trpo}  maximizes a “surrogate” objective with the trust region method
\begin{align}
    \mathcal{J}^{\mathrm{CPI}}(\theta)=\mathbb{E}_{(s_t,a_t)\sim \pi_{\theta_{\text{old}}}}\left[\frac{\pi_\theta(a_t \mid s_t)}{\pi_{\theta_\text{old}}(a_t \mid s_t)} \widehat{A}_t\right]=\mathbb{E}_{(s_t,a_t)\sim \pi_{\theta_{\text{old}}}}\left[r_t(\theta) \widehat{A}_t\right],
\end{align}
where CPI refers to conservative policy iteration \citep{kakade2002KL}. And PPO introduces clipping to this CPI objective to constrain the update of policy during maximizing this objective. \cite{chen2026exploration} studied the raw token-level surrogate $r_t(\theta) \widehat{A}_t$ in the surrogate loss function under a softmax policy update framework, particularly to characterize the correction introduced by clipping in PPO, finding that substantial upper-clipping bias in loss can lead to a wrong learning signal. Similar to policy gradient theorem \citep[cf.][]{sutton2018reinforcement}, we have
\begin{align}
    \nabla_{\theta}J^{\text{CPI}}(\theta)=\mathbb{E}_{(s_t,a_t)\sim \pi_{\theta_{\text{old}}}}\left[ r_t(\theta)\nabla_\theta \log \pi_\theta(a_t\mid s_t)\widehat{A}_t \right].
\end{align}
A natural way to control the magnitude of policy updates is to control the gradient contribution $r_t(\theta)\nabla_\theta \log \pi_\theta(a_t\mid s_t)\widehat{A}_t$ of each token. However, explicitly computing per-token gradients $\nabla_\theta \log \pi_\theta(a_t\mid s_t)$ is computationally expensive and impossible. We therefore examine whether regulating the scalar coefficient $r_t(\theta)\widehat{A}_t$ is sufficient. Our starting point of analyzing $r_t(\theta)\widehat{A}_t$ is from the gradient perspective, which is different from the objective perspective in \cite{chen2026exploration}.

Following this idea we design our ACPO objective:
\begin{equation} \label{eq: ACPO}
    \mathcal{J}^{\mathrm{ACPO}}(\theta)=\mathbb{E}_{(s_t,a_t)\sim \pi_{\theta_{\text{old}}}}\Big[\text{clip}\big(r_t(\theta)\widehat{A}_t,-\alpha,\alpha \big)\Big],
\end{equation}
where $\widehat{A}_t$ can be calculated by PPO-type advantage estimate or GRPO-type. We defer the design of the clipping range later. We next present Proposition~\ref{lower moment}, which establishes that, under mild assumptions, the gradient of the ACPO  ($G_{\mathrm{ACPO}}$) has a smaller second moment than that of PPO ($G_{\mathrm{PPO}}$), indicating smaller update magnitudes under gradient sense. The proof can be found in Section~\ref{sec: proof}.

\begin{proposition} \label{lower moment}
    Let us denote $\1_\mathrm{PPO}=\1_{\{\widehat{A}>0, r(\theta)\le 1+\epsilon\ \text{or } \widehat{A}<0, r(\theta)\geq 1-\epsilon\}}$, $\1_\mathrm{ACPO}=\1_{\{|r(\theta)\widehat{A}|\leq \alpha\}}$. Suppose that the two clipping rules retain samples
with equal probability,
{\footnotesize$\mathbb{E}[\1_{\mathrm{PPO}}]
=\mathbb{E}[\1_{\mathrm{ACPO}}]$ }, and that
{\footnotesize $
\mathbb{E}\!\left[
\left\|\nabla_\theta\log\pi_\theta(a\mid s)\right\|^2
\,\middle|\, r(\theta),\widehat{A}
\right]=C
$}
almost surely for some finite constant $C\geq 0$.
Assuming that the second moments below are finite, we have
{\footnotesize \[
\mathbb{E}\|G_{\mathrm{ACPO}}\|^2=\mathbb{E}\!\left[
\left\|
r(\theta)\widehat{A}\nabla_\theta\log\pi_\theta(a\mid s)
\1_{\mathrm{ACPO}}
\right\|^2
\right]
\leq
\mathbb{E}\!\left[
\left\|
r(\theta)\widehat{A}\nabla_\theta\log\pi_\theta(a\mid s)
\1_{\mathrm{PPO}}
\right\|^2
\right]=\mathbb{E}\|G_{\mathrm{PPO}}\|^2.
\] }
\end{proposition}

\begin{remark}
    The proof does not rely on the specific form of the PPO
clipping rule. The same comparison holds for any binary
retention rule measurable with respect to
$(r(\theta),\widehat{A})$ that has the same retention
probability as ACPO. Furthermore, if we assume that the gradient estimator after clipping has similar(same) norm {\small $\|\mathbb{E}[r(\theta)\widehat{A}\nabla_{\theta}\log \pi_{\theta}(a\mid s)\1_\mathrm{ACPO}]\| \approx (=)  \|\mathbb{E}[r(\theta)\widehat{A}\nabla_{\theta}\log \pi_{\theta}(a\mid s)\1_\mathrm{PPO}]\|$}, for example when
both are unbiased estimators of the same target gradient, then Proposition~\ref{lower moment} implies
\[
\operatorname{Var}(G_{\mathrm{ACPO}})
\leq
\operatorname{Var}(G_{\mathrm{PPO}}), \mathrm{by} \operatorname{Var}(G)
=
\mathbb{E}\!\left[
\|G-\mathbb{E}[G]\|^2
\right]
=
\mathbb{E}\!\left[\|G\|^2\right]
-\|\mathbb{E}[G]\|^2. 
\] Hence Proposition~\ref{lower moment} with additional assumption provides a variance-reduction guarantee of ACPO.
\end{remark}

The quality and difficulty of training data play an important role in RL post-training. Prior studies have shown that prompts of intermediate difficulty provide more effective training signals than prompts that are either too easy or too difficult \citep{zhang2026speedrlfastertrainingreasoning,bae2026onlinedifficultyfilteringreasoning,gao2026promptcurriculumlearningefficient}. However recent work only focuses on prompt selection, for example, \cite{gao2026promptcurriculumlearningefficient,zhang2026speedrlfastertrainingreasoning} introduce adaptive prompt selection and curriculum learning to improve training efficiency. How to account for prompt difficulty directly in the policy optimization algorithm remains less explored. One mechanism linking prompt difficulty to policy optimization is its influence on advantage estimates. Motivated by this connection, we track the distribution of the token-level gradient coefficient $r_t(\theta)\widehat{A}_t$ across prompts of varying difficulty throughout training in Figure~\ref{fig:prompt}.

\begin{figure*}[ht]
    \centering

    \begin{minipage}[t]{0.32\textwidth}
        \centering
        \includegraphics[width=0.95\linewidth]{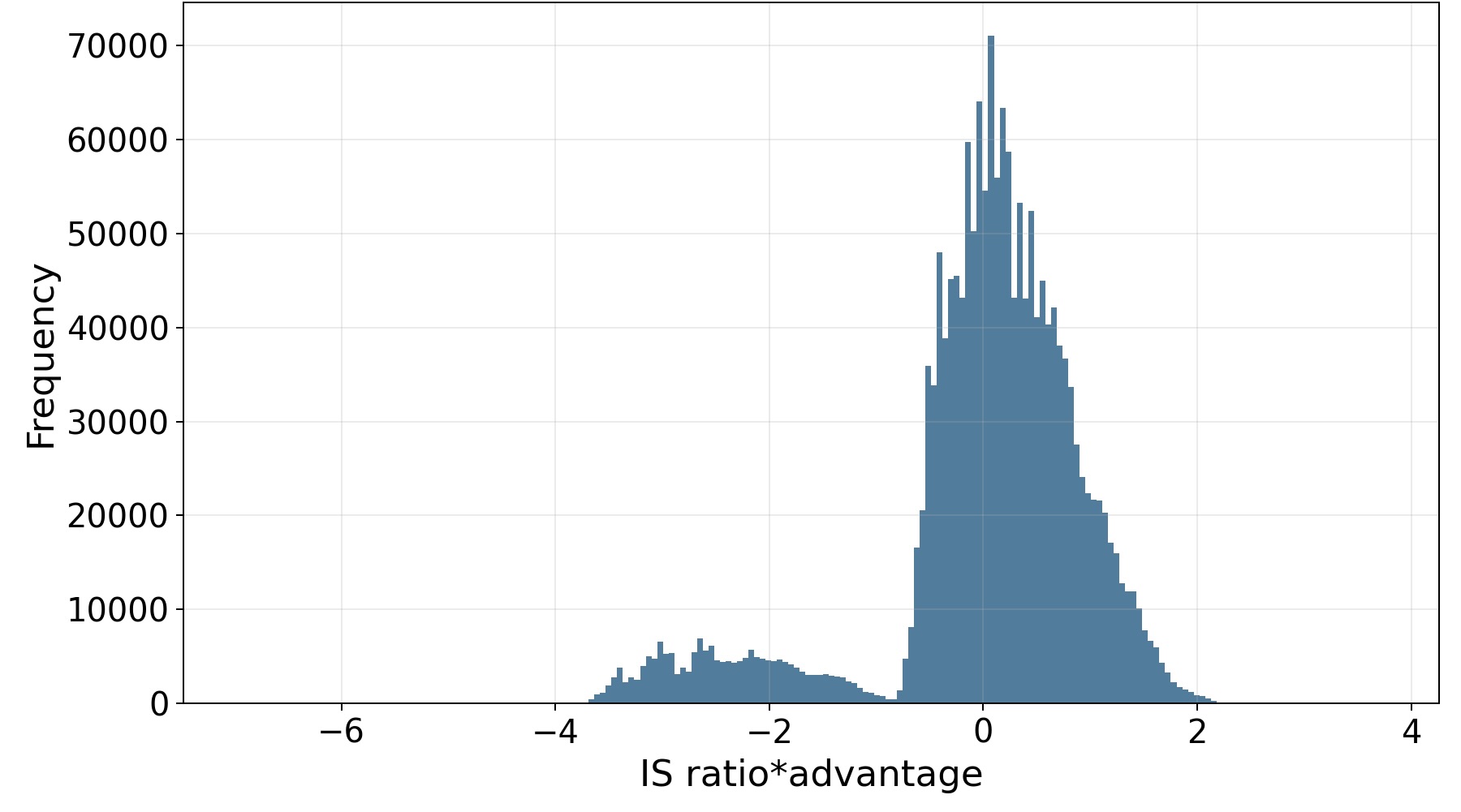}
        \centerline{\small (a) Easy prompt}
    \end{minipage}
    \hfill
    \begin{minipage}[t]{0.32\textwidth}
        \centering
        \includegraphics[width=0.95\linewidth]{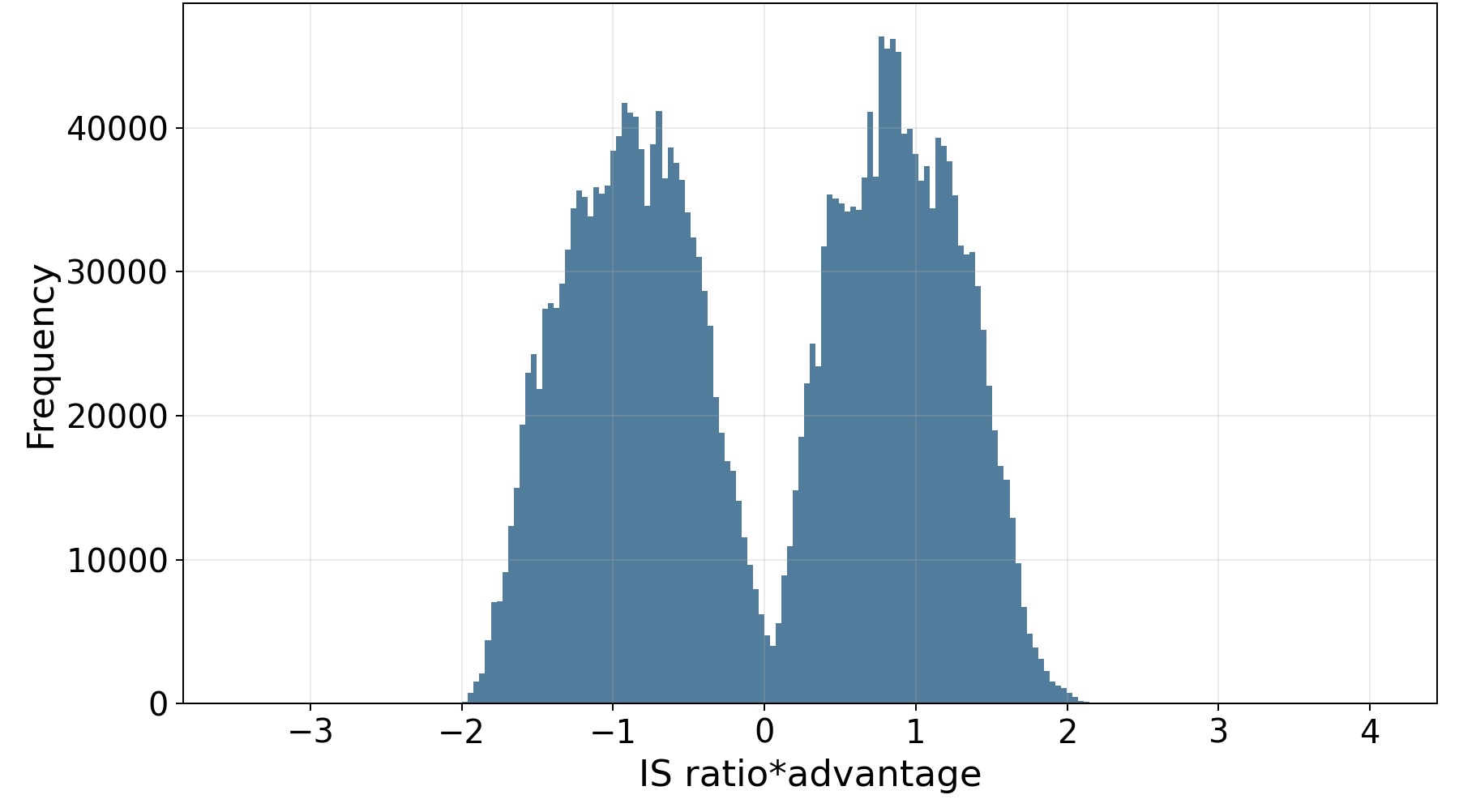}
        \centerline{\small (b) Medium prompt}
    \end{minipage}
    \hfill
    \begin{minipage}[t]{0.32\textwidth}
        \centering
        \includegraphics[width=0.95\linewidth]{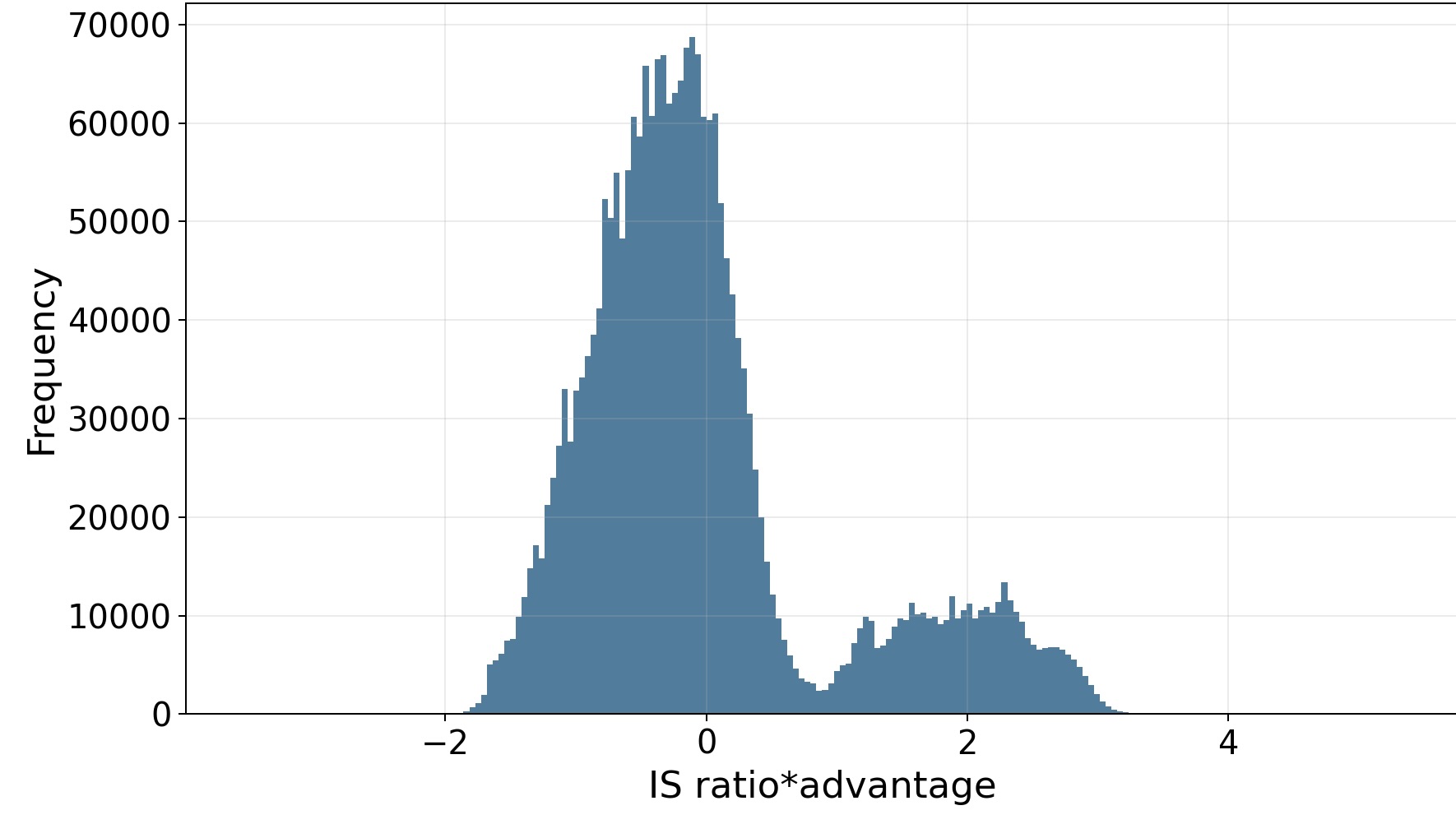}
        \centerline{\small (c) Hard prompt}
    \end{minipage}

    \caption{$r_t(\theta)\widehat{A}_t$ across prompts }
    \label{fig:prompt}
\end{figure*}
As shown in Figure~\ref{fig:prompt}, a clearer pattern emerges when comparing the tails of the token-level gradient coefficient \(r_t(\theta)\widehat{A}_t\) across difficulty levels. For hard prompts (accuracy 0–30\%), the distribution exhibits a pronounced positive tail, indicating that a subset of tokens receives relatively large positive gradient coefficients and therefore contributes disproportionately to reinforcing the current policy. In contrast, for easy prompts (accuracy 70–100\%), the distribution develops a substantially heavier negative tail, suggesting that easy prompts are more likely to produce large negative updates for a small fraction of tokens. Meanwhile, medium-difficulty prompts (accuracy 30–70\%) show a more concentrated distribution with substantially lighter tails, implying that their token-level updates are comparatively more moderate and stable. Following the principle that \textbf{prompts of intermediate difficulty provide more effective training signals}, we introduce a symmetric clipping range $[-\alpha,\alpha]$ for \(r_t(\theta)\widehat{A}_t\) to mitigate the influence of the long tail induced by overly easy or overly difficult prompts. This clipping encourages token-level updates to remain within a moderate regime, making their optimization behavior more similar to that observed for prompts of intermediate difficulty.

\section{Connection to Gradient Clipping in Policy Mirror Descent}
Clipping is widely used in RL post-training, with numerous variants proposed to stabilize policy optimization, while gradient clipping in optimization algorithms likewise plays a central role in stabilizing large-model training. These observations raise a natural question: is there a connection between these two forms of clipping? We establish this connection through policy mirror descent, showing that applying gradient clipping within this framework yields the clipping objective of ACPO. This derivation links gradient clipping in optimization to clipping in post-training. 

\begin{algorithm}
\caption{The policy mirror descent (PMD) method with gradient clipping}
\label{alg:pmd}
\begin{algorithmic}[1]
\State \textbf{Input:} initial points $\pi_0$, $\pi_0(a\mid s)=1/|\mathcal{A}|$, clipping threshold $\alpha$ and stepsizes $\eta_k \ge 0$.
\For{$k = 0,1,\ldots$}
    \State
    {\footnotesize \[
    \pi_{k+1}(\cdot \mid s)
    =
    \underset{p(\cdot \mid s)\in \Delta(\mathcal A)}{\operatorname{arg\,min}}
    \left\{
    \eta_k\frac{\alpha}{
        \max\{\alpha,\|A^{\pi_k}(s,\cdot)\|\}
    }
    \left[
    \left\langle A^{\pi_k}(s,\cdot),p(\cdot\mid s)\right\rangle
    + h^p(s)
    \right]
    +
    D_{\pi_k}^{p}(s)
    \right\}, \forall s\in\mathcal S
    \]}
\EndFor
\end{algorithmic}
\end{algorithm}

Introducing gradient clipping in (\ref{alg: PMD}) gives us Algorithm~\ref{alg:pmd}, because $A^{\pi_k}(s,\cdot)$ serves as gradient in the update. Then starting from PMD with gradient clipping and $h^p=0$, we first apply importance sampling to express the advantage term using samples from the old policy, introducing the importance ratio $r_\theta(s,a)=\frac{\pi_\theta(a\mid s)}{\pi_{\mathrm{old}}(a\mid s)}$. Then we restrict the optimization to a parameterized policy class, approximately solve the resulting subproblem using multiple stochastic gradient steps (see Section~\ref{sec:transformation} for details). Finally, we remove the explicit KL proximity penalty arising from the  mirror map, obtaining ACPO objective~\ref{eq: ACPO}. Omitting the KL penalty isolates the role of clipping and enables us to investigate whether ACPO can maintain stable training without explicit KL regularization. Moreover, removing the KL penalty is consistent with the experimental settings in recent papers \citep{yu2025dapo,xiong2025reinforceadaadaptivesamplingframework}.

This construction connects ACPO to the clipping principle underlying PMD while introducing a different clipping mechanism. During the translation of the clipping principle into a sample-based surrogate, we replace PMD's statewise advantage rescaling with direct clipping of the importance-weighted advantage $r_t(\theta)\widehat A_t$. Differentiating this surrogate preserves each sample's gradient contribution within the clipping interval and suppresses it outside the interval. 

Now, we present convergence results of Algorithm~\ref{alg:pmd} and proofs are presented in Section~\ref{sec: convergence rate}. Here, we take the Bregman Divergence induced by negative entropy that $D_{\pi}^{\pi'}=\KL (\pi' ||\pi)$. Following \cite{pmdlan}, let $\nu^*$ be the stationary state distribution induced by the optimal policy $\pi^*$, we consider the convergence criterion of the following optimization problem with specific distribution $\nu^*$:
\begin{align*}
\min_\pi f(\pi)&:=\mathbb E_{s\sim\nu^*}[V^\pi(s)]\quad \mathrm{s.t.} ~\pi(\cdot\mid s) \in \Delta(\mathcal{A}), \forall s \in \mathcal{S},
\end{align*} 
then we define $F_k=f(\pi_k)-f(\pi^*)=E_{s\sim \nu^*}[V_{\pi_k}(s)]-E_{s\sim \nu^*}[V_{\pi^*}(s)]$ and $\mathcal{D}_k=E_{s\sim \nu^*}[D^{\pi^*}_{\pi_k}(s)]$. We analyze convergence under constant and adaptive step sizes,
both with strongly convex regularization and without regularization. We show that Algorithm~\ref{alg:pmd} can achieve a linear convergence rate
for solving RL problems with strongly convex regularizers and  a sublinear rate without it.
For convenience, we define
\[
\lambda_k(s):=
\frac{\alpha}{
    \max\{\alpha,\|A^{\pi_k}(s,\cdot)\|\}
},
\qquad
\kappa_0:=\min_{s\in\mathcal S}\lambda_0(s).
\]

\begin{theorem}[Convergence with strongly convex regularizer]\label{thm: stongly convex}Suppose $h^p$ is $\mu$-strongly convex with $\mu>0$, $D_{\pi}^{\pi'}=\KL (\pi' ||\pi)$, cost function $c(s,a)$ and $h^p$ are bounded.
Then the following bounds hold for all $N\geq 1$:
\begin{enumerate}
    \item \textbf{Constant step sizes.}
    There exists
    a constant $1\ge\kappa>0$, for $\eta_k=\eta$ for all $k\geq 0$, we have
    \[
    F_N+(\mu+\frac{1}{\eta})\mathcal D_N
    \leq
    \rho^N
    \left[
        F_0+(\mu+\frac{1}{\eta})\log|\mathcal A|
    \right],
    \]
    where $
    \rho=
    \max\left\{
        \gamma,\frac{1}{\kappa(1+\eta\mu)}
    \right\}
    $. If $\eta\mu>1/\kappa-1$, then $\rho<1$ .

    \item \textbf{Adaptive step sizes.}
    For $\eta_0>0$, choose 
    $\eta_{k+1}
    =
    \eta_k
    \max\left\{
        1,\,
        \max_{s\in\mathcal S}
        \frac{\lambda_k(s)}{\lambda_{k+1}(s)}
    \right\}$.
    Then
    \[
    F_N
    \leq
    \rho_0^N
    \left[
        F_0+
        \left(
            \mu+\frac{1}{\eta_0\kappa_0}
        \right)
        \log|\mathcal A|
    \right],
    \]
    where
    $
    \rho_0=
    \max\left\{
        \gamma,\frac{1}{1+\mu\eta_0\kappa_0}
    \right\}<1.
    $
\end{enumerate}
\end{theorem}

\begin{theorem}[Convergence without regularizer] \label{thm: nonconvex}
Suppose $h^p=0$, $D_{\pi}^{\pi'}=\KL (\pi' ||\pi)$,  cost function $c(s,a)$ and $h^p$ are bounded. Then the following bounds hold for all
$N\geq 1$:
\begin{enumerate}
    \item \textbf{Constant step sizes.}
    There exists
    a constant $C_0>0$, for $\eta_k=\eta$ for all $k\geq 0$, we have
    \[
    F_N
    \leq
    \frac{C_0}{(1-\gamma)N}
    \left[
        \gamma F_0+
        \frac{\log|\mathcal A|}{\eta\kappa_0}
    \right].
    \]

    \item \textbf{Adaptive step sizes.}
    For $\eta_0>0$, choose
    $
    \eta_{k+1}
    =
    \eta_k
    \max\left\{
        1,\,
        \max_{s\in\mathcal S}
        \frac{\lambda_k(s)}{\lambda_{k+1}(s)}
    \right\}.
    $
    Then
    \[
    F_N
    \leq
    \frac{1}{(1-\gamma)N}
    \left[
        \gamma F_0+
        \frac{\log|\mathcal A|}{\eta_0\kappa_0}
    \right].
    \]
\end{enumerate}
\end{theorem}

\section{Experiments}
\label{sec:experiments}

\subsection{Experimental Setup.}
In this section, we present experiments to evaluate the performance of ACPO on reasoning tasks. We incorporate PPO and GRPO with ACPO and denote it as PPO-AC and GRPO-AC. We compare
GRPO-AC and PPO-AC against two standard IS ratio clipping methods: GRPO and PPO.
\paragraph{Datasets and models.} We train our models using the training set from DAPO \citep{yu2025dapo}, which comprises approximately 17.4k math problems sourced from the AoPS website and
official competition homepages. We
employ the Math-Verify tool \citep{kydlicek2025mathverify} for automatic solution correctness verification. To demonstrate the generality of our method across model scales, we conduct experiments using \textbf{Qwen2.5-Math-7B} and \textbf{Qwen3(1.7B, 4B and 8B)}.

\paragraph{Evaluation.} We assess the models’ reasoning abilities on several standard mathematical reasoning benchmarks: \textbf{MATH500} \citep{hendrycks2021math500}, \textbf{Minerva Math} \citep{Lewkowycz2022Minerva}, \textbf{OlympiadBench} \citep{he2024olympiadbench} and \textbf{AIME-like} \citep{xiong2025reinforceadaadaptivesamplingframework}, which consists of 230 problems from recent competitions: AIME24, AIME25, HMMT24, HMMT25, BRUMO25, AMC23, and CMIMC25. All evaluations report the Pass@1 accuracy
averaged over 16 or 32 samples, generated with a temperature of 1.0, top-p=1 and a max token limit of 8196.

\paragraph{RL Training Details.} We conduct all experiments using the VERL reinforcement learning framework \citep{Sheng2025verl}. In PPO, we generate one response per prompt ($n=1$), for GRPO we generate four responses per prompt ($n=4$) for Qwen3 and eight responses per prompt ($n=8$) for Qwen2.5-Math-7B. Following prior work \citep{schulman2017ppo,yu2025dapo}, we set the standard clipping range to $0.2$ for PPO and $[0.2, 0.28]$ for GRPO. For PPO-AC, we set $\alpha=3$ in all experiments. For GRPO-AC, we set $\alpha=3$ for Qwen2.5-Math-7B and $\alpha=2$ for Qwen3. More details of our training setting can be found in Appendix~\ref{training details}.

\subsection{Main Results}

We conduct experiments on Qwen2.5-Math-7B and Qwen3 (1.7B, 4B and 8B). We consider the weighted accuracy(weighted by problem amount on each benchmark).The evaluation performance curves and best
evaluation results throughout
training are presented in Table~\ref{tab:results} and Figure~\ref{fig:main_result}, respectively. The effectiveness of our method is demonstrated in Table~\ref{tab:results}. Across four model configurations and four standard mathematical reasoning benchmarks, incorporating ACPO consistently improves both PPO and GRPO. In particular, PPO-AC increases the weighted accuracy over PPO by 1.6, 2.4, 2.7, and 4.0 percentage points on Qwen2.5-7B-Math, Qwen3-1.7B, Qwen3-4B, and Qwen3-8B, respectively, demonstrating increasingly pronounced gains on stronger models. Similar improvements are observed for GRPO-AC, which consistently outperforms its corresponding GRPO baseline in weighted accuracy. The gains are especially substantial on the more challenging OLYMPIAD and AIME-like benchmarks. For example, on Qwen3-8B, PPO-AC improves OLYMPIAD accuracy from 63.7\% to 70.5\% and AIME-like accuracy from 43.6\% to 48.8\%.  These results demonstrate that ACPO provides robust and broadly applicable improvements across different policy optimization algorithms, model families, and model scales.

Figure~\ref{fig:main_result} illustrates the step-wise weighted accuracy averaged across all benchmarks for ACPO and the baseline methods. Although all approaches exhibit improved reasoning performance as training progresses, ACPO consistently converges faster and achieves higher  performance than baselines. In particular, GRPO-AC maintains the best performance throughout the entire training process. Moreover, PPO-AC consistently outperforms PPO and achieves performance comparable to GRPO, while GRPO requires multiple rollouts per prompt.

\begin{table*}[h]
\centering
\caption{Performance of methods on math benchmarks, where we choose the best weighted accuracy result during training. Bold marks the better result within each matched PPO or GRPO pair.}
\label{tab:results}
\small
\setlength{\tabcolsep}{4pt}
\begin{tabular}{lrrrrr}
\toprule
Method & Weighted Acc. & MATH500 & MINERVA & OLYMPIAD & AIME-like \\
\midrule
\multicolumn{6}{l}{\textbf{Qwen2.5-7B-Math} (avg@32)} \\
PPO & 48.4& 79.0& 33.0& 41.2& 21.4\\
\rowcolor{blue!8}
PPO-AC & \textbf{50.0} & \textbf{80.5}& \textbf{35.3}&
\textbf{42.8}& \textbf{21.9}\\
GRPO & 51.2& 82.0& 37.2& 43.2& 24.1\\
\rowcolor{blue!8}
GRPO-AC & \textbf{52.6}& \textbf{82.8}& \textbf{37.8}&
\textbf{45.6}& \textbf{24.9}\\
\midrule
\multicolumn{6}{l}{\textbf{Qwen3-1.7B} (avg@16)} \\
PPO & 57.3& 86.2& 39.2& 52.8& 28.8\\
\rowcolor{blue!8}
PPO-AC & \textbf{59.7}& \textbf{88.1}& \textbf{39.7}&
\textbf{56.9}& \textbf{30.1}\\
GRPO & 60.6& 88.2& 39.5& 58.9& 30.8\\
\rowcolor{blue!8}
GRPO-AC & \textbf{62.1}& \textbf{89.4}& \textbf{40.3}&
\textbf{60.3}& \textbf{34.3}\\
\midrule
\multicolumn{6}{l}{\textbf{Qwen3-4B} (avg@16)} \\
PPO & 66.7& 93.3& 45.4& 63.8& 43.3\\
\rowcolor{blue!8}
PPO-AC & \textbf{69.4}& \textbf{94.0}& \textbf{46.4}&
\textbf{68.3}& \textbf{46.3}\\
GRPO & 70.0& 94.7& 46.6& 69.2& 46.4\\
\rowcolor{blue!8}
GRPO-AC & \textbf{70.6}& \textbf{94.8}& \textbf{46.8}&
\textbf{69.8}& \textbf{48.5}\\
\midrule
\multicolumn{6}{l}{\textbf{Qwen3-8B} (avg@16)} \\
PPO & 67.3& 93.1& 48.6& 63.7& 43.6\\
\rowcolor{blue!8}
PPO-AC & \textbf{71.3}& \textbf{94.8}& \textbf{48.9}&
\textbf{70.5}& \textbf{48.8}\\
GRPO & 70.6& \textbf{95.0}& 48.7& 69.1& 47.3\\
\rowcolor{blue!8}
GRPO-AC & \textbf{71.4}& \textbf{95.0}& \textbf{49.4}&
\textbf{70.4}& \textbf{49.0}\\
\bottomrule
\end{tabular}
\end{table*}

\begin{figure}[h]
    \centering
    \begin{minipage}{0.9 \linewidth}
        \centering

        \begin{minipage}[t]{0.49\linewidth}
            \centering
            \includegraphics[width=\linewidth]{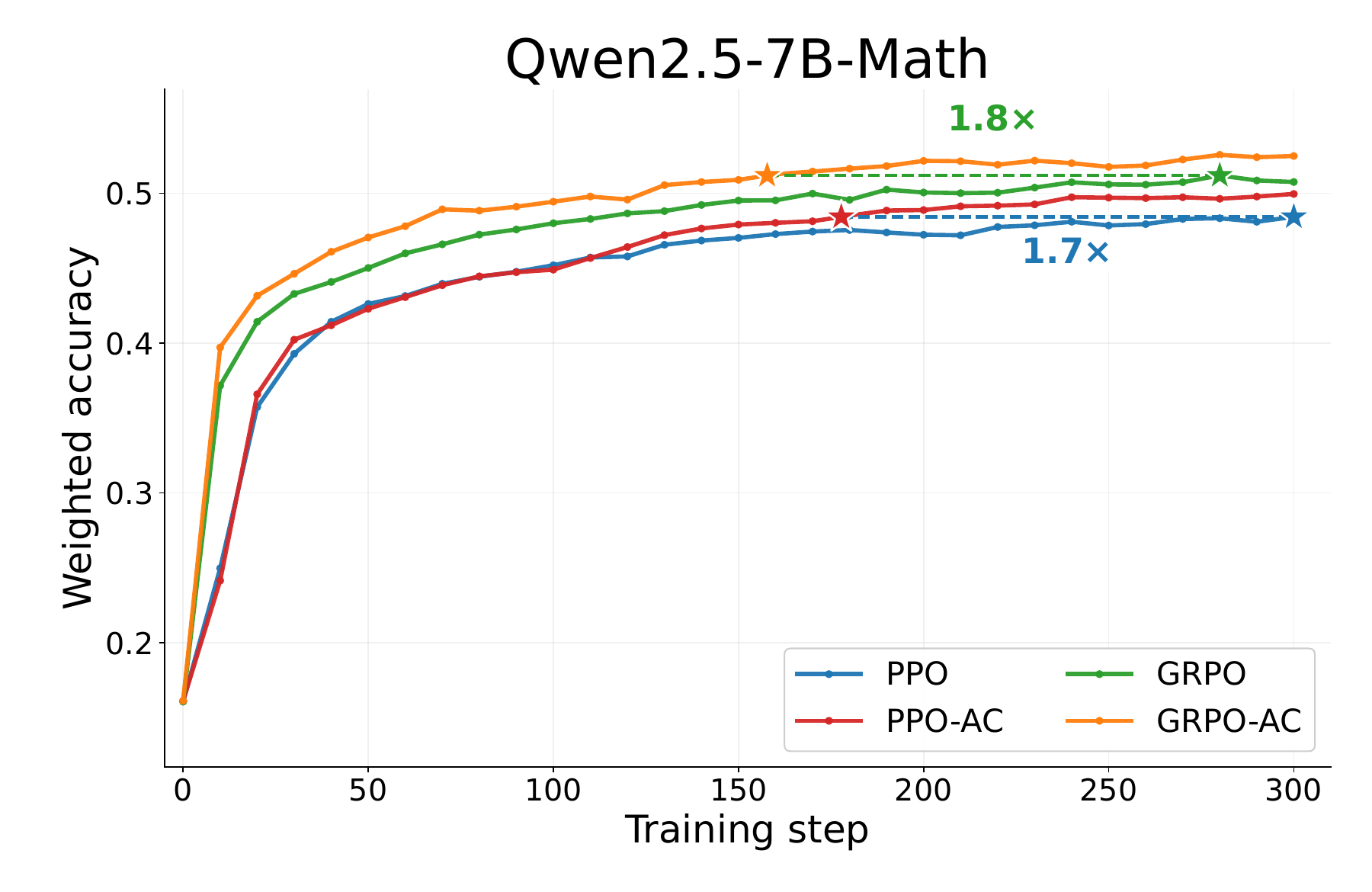}
        \end{minipage}\hfill
        \begin{minipage}[t]{0.49\linewidth}
            \centering
            \includegraphics[width=\linewidth]{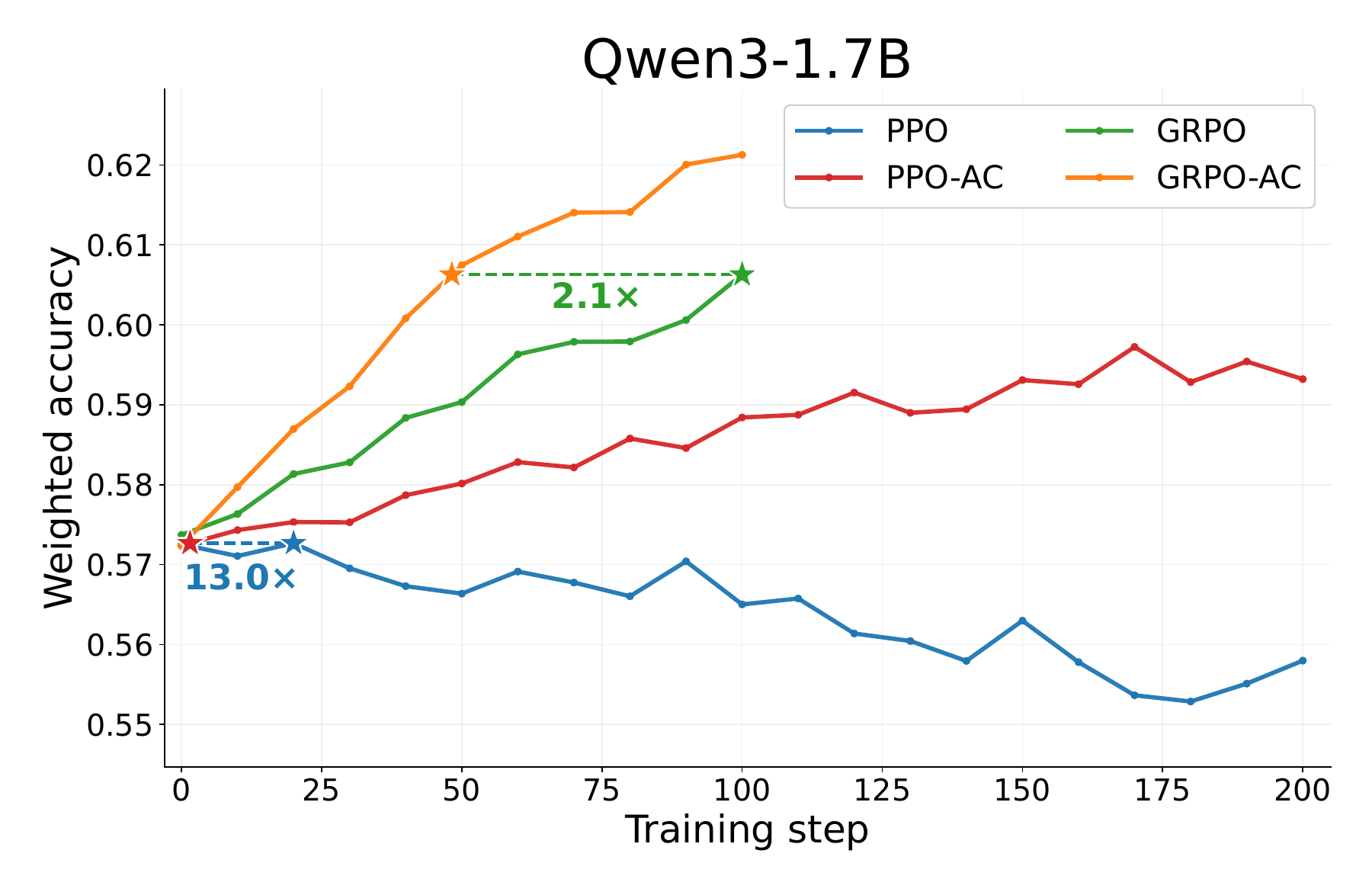}
        \end{minipage}

        \par\smallskip

        \begin{minipage}[t]{0.49\linewidth}
            \centering
            \includegraphics[width=\linewidth]{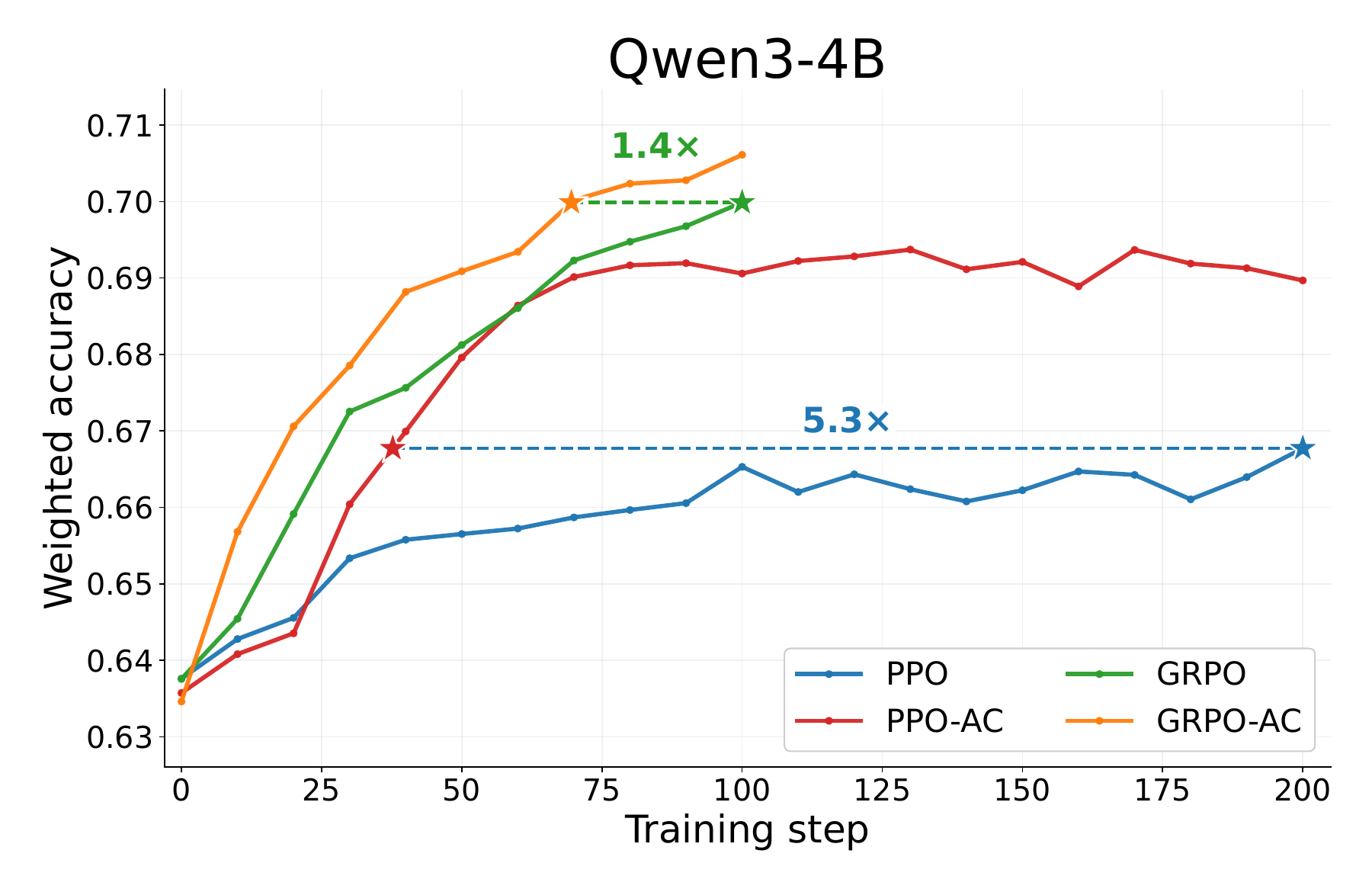}
        \end{minipage}\hfill
        \begin{minipage}[t]{0.49\linewidth}
            \centering
            \includegraphics[width=\linewidth]{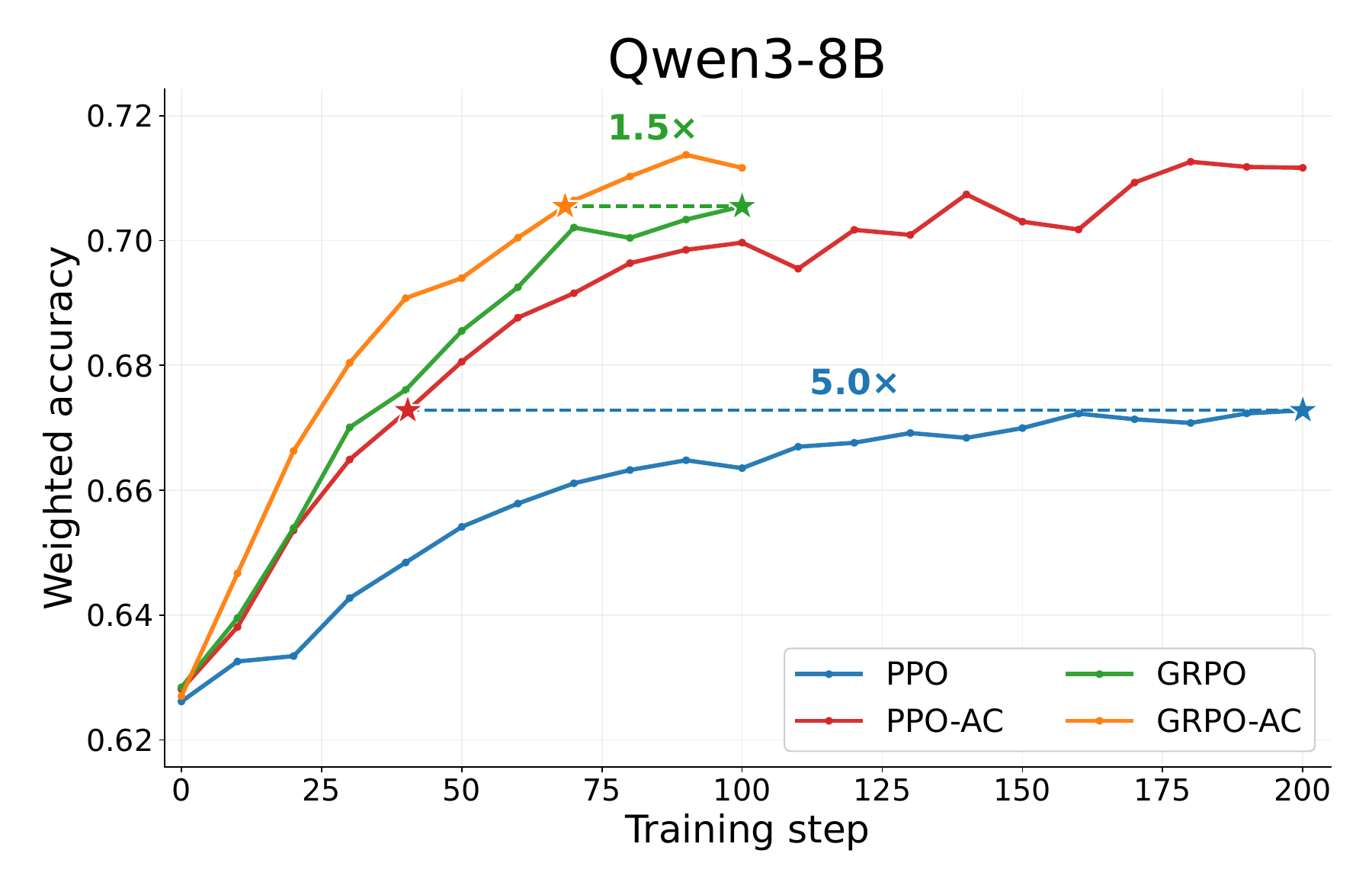}
        \end{minipage}
    \end{minipage}

    \caption{Accuracy during Training Process} 
    \label{fig:main_result}
\end{figure}

\subsection{Training Dynamics and  Analysis}
\paragraph{Training efficiency. }
Our clipping mechanism in ACPO stabilizes policy optimization from both policy-gradient and mirror-descent perspectives while making more effective use of the training signal. As shown in Figure~\ref{fig:main_result}, ACPO achieves average training-step speedups of 6.25$\times$ over PPO and 1.7$\times$ over GRPO, measured by the steps required to match each baseline’s best accuracy. For Qwen3-4B and Qwen3-8B, PPO-AC matches the best accuracy attained by PPO over 200 steps in approximately 40 steps (5.3$\times$ and 5.0$\times$). Especially on Qwen3-1.7B, PPO yields little initial improvement and its accuracy declines after 20 steps, whereas PPO-AC continues to improve model accuracy, suggesting greater robustness in this setting. These results confirm that ACPO improves accuracy while accelerating training speed, with larger efficiency gains over PPO than over GRPO. This difference arises from how the two methods estimate advantages: PPO uses token-level advantages derived from a value model, whereas GRPO assigns the same advantage to all tokens within a sequence. This finer-grained signal allows PPO to benefit more from ACPO’s clipping mechanism.

\paragraph{Entropy.}
Vanilla PPO and GRPO can exhibit rapid entropy decay, limiting exploration and prematurely concentrating the policy on a narrow set of reasoning trajectories~\citep{yu2025dapo,yue2025vapo,cui2025entropy}. The standard clipping mechanism in PPO-style objectives further exacerbates this issue by constraining probability increases for low-probability exploratory tokens.  And \cite{chen2026exploration} provides a theoretical analysis on clipping range and entropy. As shown in Figure~\ref{fig:entropy}, ACPO substantially alleviates entropy collapse during training. The policy entropy of GRPO-AC remains relatively stable and fluctuates around its initial value, whereas that of  GRPO increases rapidly with the use of the Clip-Higher technique\citep{yu2025dapo}. Although the entropy of PPO-AC still declines during training, its decrease is considerably slower than that of PPO. Our clipping mechanism in ACPO preserves the contributions of tokens with relatively small advantages, which helps sustain exploration and policy diversity.

\begin{figure}[htbp]
    \centering
    \begin{subfigure}[b]{0.24\textwidth}
        \centering
        \includegraphics[width=\linewidth]{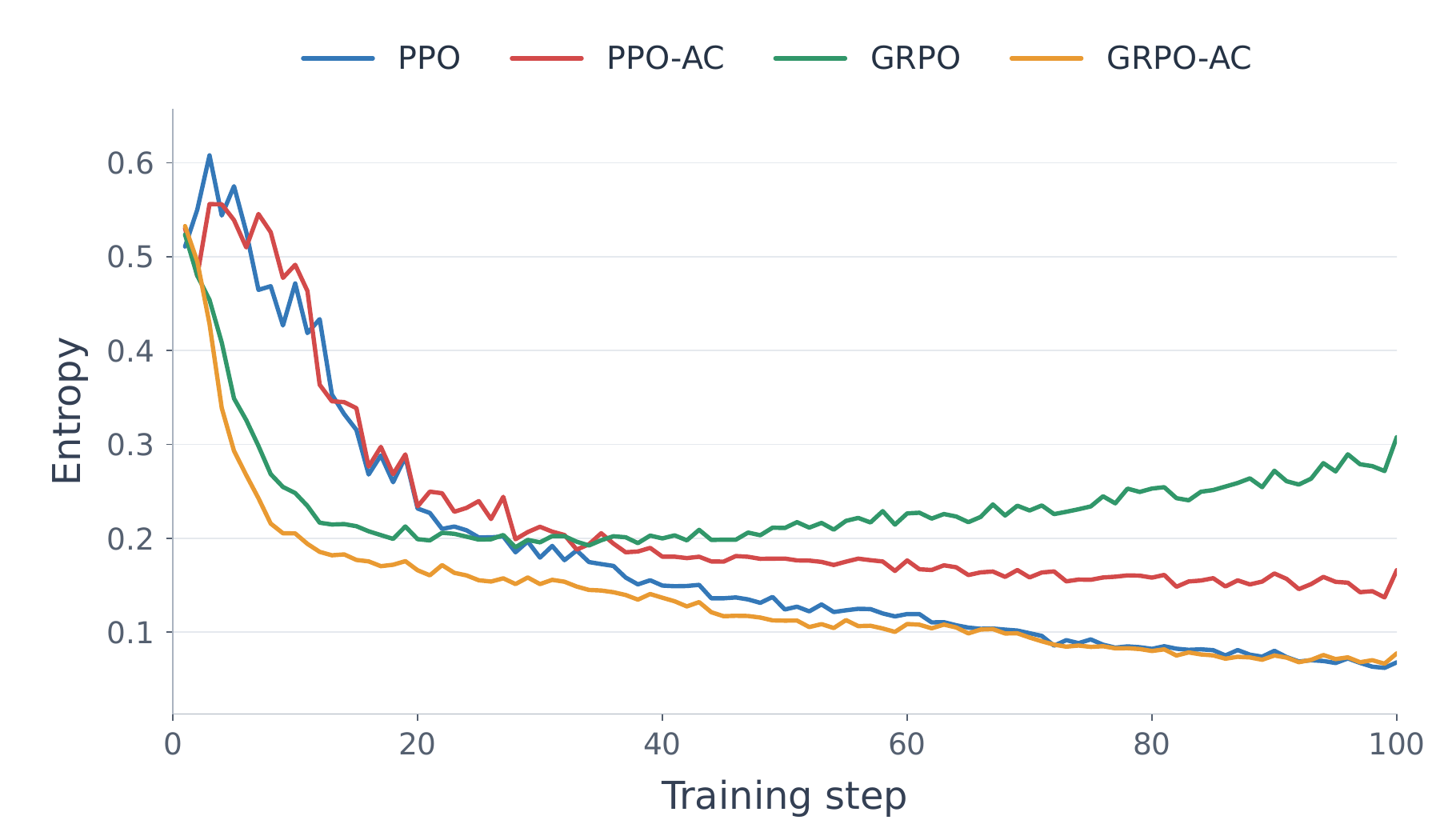}
        \caption{Qwen2.5-7B-Math}
        \label{fig:entropy-7b}
    \end{subfigure}\hfill
    \begin{subfigure}[b]{0.24\textwidth}
        \centering
        \includegraphics[width=\linewidth]{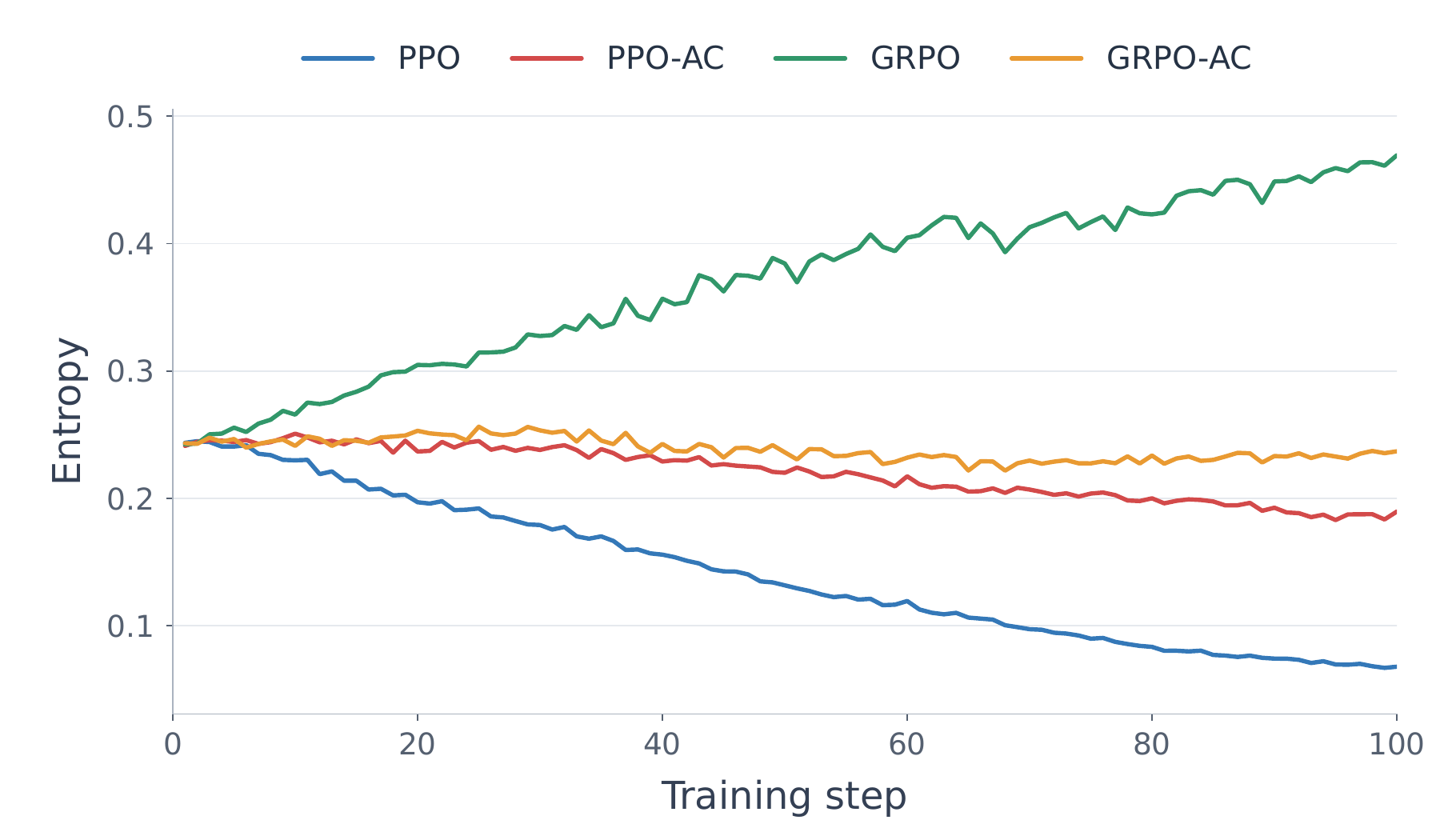}
        \caption{Qwen3-1.7B}
        \label{fig:entropy-1.7b}
    \end{subfigure}\hfill
    \begin{subfigure}[b]{0.24\textwidth}
        \centering
        \includegraphics[width=\linewidth]{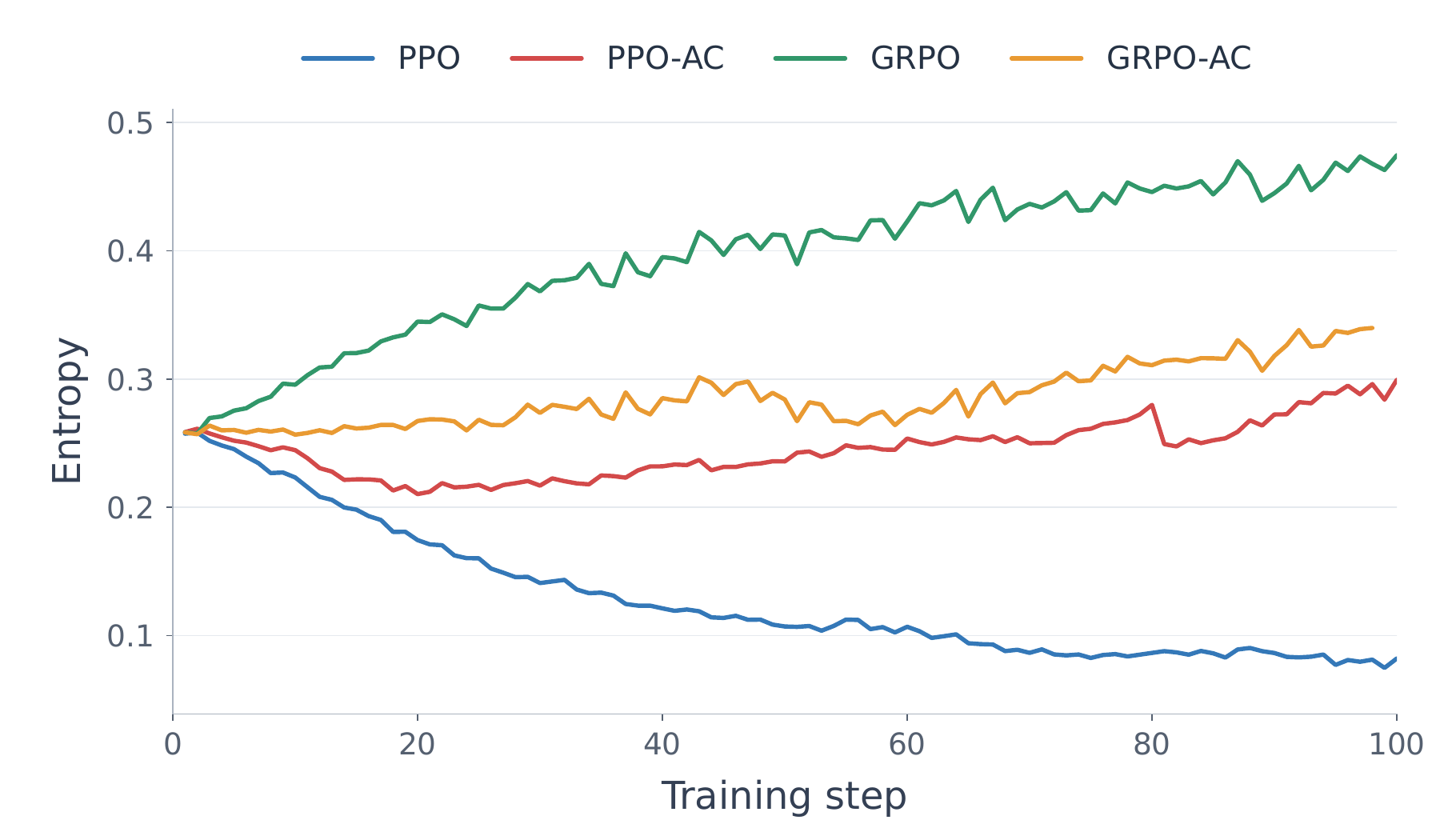}
        \caption{Qwen3-4B}
        \label{fig:entropy-4b}
    \end{subfigure}\hfill
    \begin{subfigure}[b]{0.24\textwidth}
        \centering
        \includegraphics[width=\linewidth]{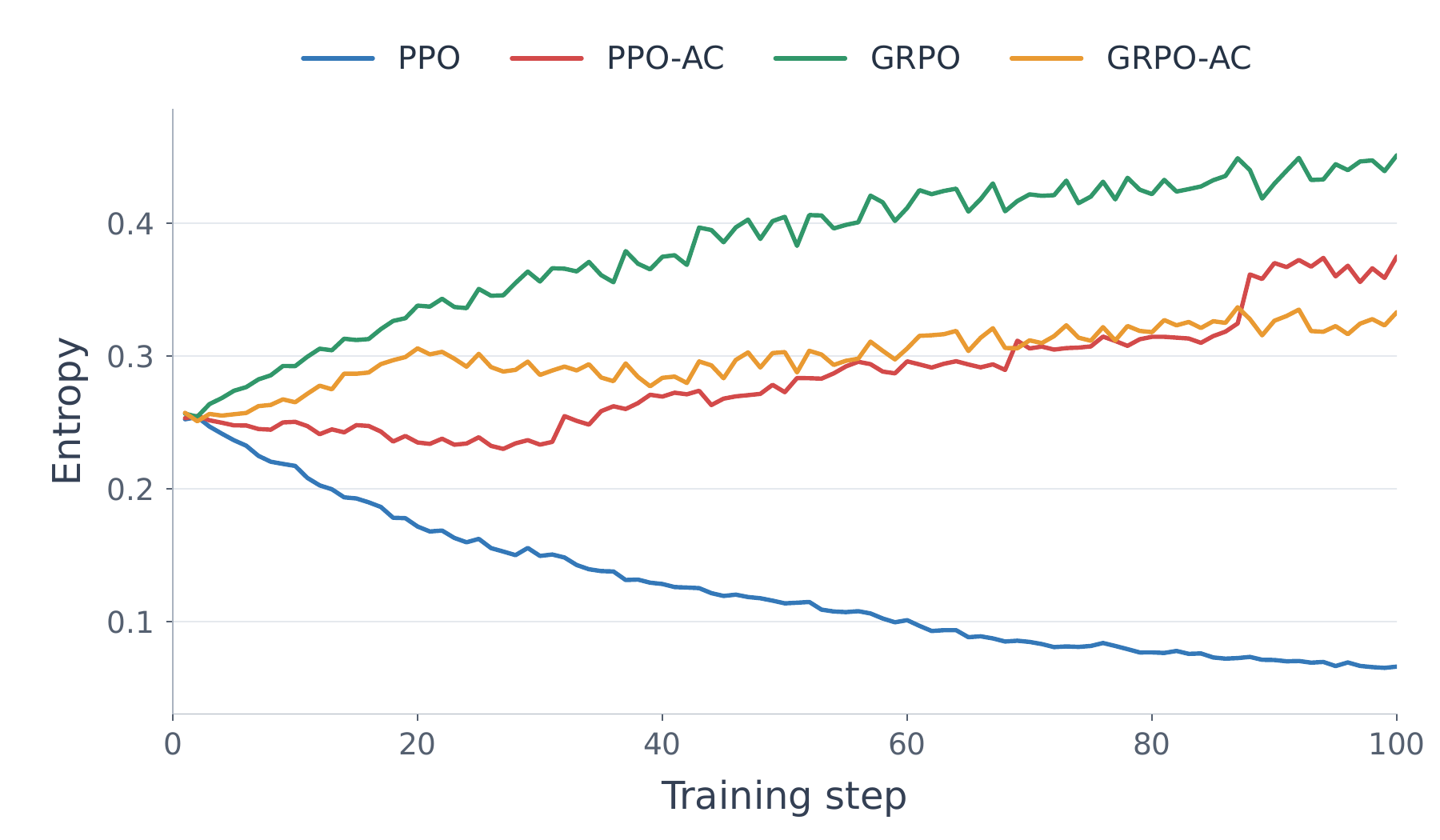}
        \caption{Qwen3-8B}
        \label{fig:entropy-8b}
    \end{subfigure}
    \caption{Policy entropy over the first 100 training steps
    across four models.}
    \label{fig:entropy}
\end{figure}

\section{Related Work}
\paragraph{Mirror Descent for Policy Optimization.}
Mirror descent provides a principled framework for constrained optimization and has been widely applied to policy optimization. \citet{tomar2022mirror} propose mirror descent policy optimization (MDPO), which updates the policy by approximately solving the policy mirror descent update. The derived On-Policy MDPO objective is similar to the unclipped surrogate with KL regularizer in PPO.  \cite{chen2026exploration} also establish the connection between optimizing the unclipped KL-regularized surrogate objective and exponential-step token-level policy updates through MD. \citet{song2026bellmanpolicyoptimization} use the Bellman equation to reformulate policy mirror descent (PMD) as a trajectory-level objective and derive a practical BPO loss through approximation. While these studies derive practical objectives from mirror descent, our work takes a novel perspective: using PMD to understand and motivate clipping mechanisms in policy optimization.

\paragraph{Clipping Mechanisms.}
Clipping mechanisms have received considerable attention in policy optimization and post-training \citep{yu2025dapo,minimax2025CISPO,jia2026gapo,zheng2025gspo, cai2026RIPO}. These methods primarily modify how importance ratios are clipped. DAPO introduces Clip-Higher, using clipping range $[0.2,0.28]$ to encourage exploration \citep{yu2025dapo}. CISPO clips the importance-sampling weights while retaining the corresponding gradient contributions \citep{minimax2025CISPO}. From token level to sequence level, \cite{zheng2025gspo} design clipping over trajectory, while GAPO adapts the IS clipping range to trajectory advantages in GRPO \citep{jia2026gapo}.  

\section{Conclusion}
We introduced Advantage Clipped Policy Optimization (ACPO), which directly clips the product of IS ratio and advantage to stabilize policy optimization. ACPO can be readily integrated into PPO and GRPO, admits an interpretation through gradient-clipped policy mirror descent, and consistently improves accuracy and training efficiency across multiple Qwen models and mathematical-reasoning benchmarks. These results demonstrate that advantage clipping is a simple and effective alternative to conventional IS-ratio clipping for reinforcement-learning post-training of LLMs.

\bibliographystyle{plainnat}
\bibliography{mybib}

\begin{center}
    {\LARGE\bfseries Appendix}
\end{center}

\appendix

\setcounter{theorem}{0}
\setcounter{proposition}{0}
\setcounter{lemma}{0}
\setcounter{remark}{0}

\renewcommand{\thetheorem}{A.\arabic{theorem}}
\renewcommand{\theproposition}{A.\arabic{proposition}}
\renewcommand{\thelemma}{A.\arabic{lemma}}
\renewcommand{\theHtheorem}{A.\arabic{theorem}}
\renewcommand{\theHproposition}{A.\arabic{proposition}}
\renewcommand{\theHlemma}{A.\arabic{lemma}}
\renewcommand{\theHremark}{A.\arabic{remark}}

In Section~\ref{sec: proof} and \ref{sec: convergence rate}, we provide detailed proofs of the proposition and theorem in our main pages. Section~\ref{sec:transformation} discusses how to derive our ACPO objective from PMD with gradient clipping step by step. And in Section~\ref{Additional Experimental  Details}, we present more experiment results and settings.

\section{Proof of Proposition~\ref{lower moment}} \label{sec: proof}
\begin{proposition}[cf. Proposition~\ref{lower moment}] \label{lower moment app}
    Let we denote $\1_\mathrm{PPO}=\1_{\{\widehat{A}>0, r(\theta)\le 1+\epsilon\ \text{or } \widehat{A}<0, r(\theta)\geq 1-\epsilon\}}$, $\1_\mathrm{ACPO}=\1_{\{|r(\theta)\widehat{A}|\leq \alpha\}}$. Suppose that the two clipping rules retain samples
with equal probability,
$\mathbb{E}[\1_{\mathrm{PPO}}]
=\mathbb{E}[\1_{\mathrm{ACPO}}]$ , and that
$
\mathbb{E}\!\left[
\left\|\nabla_\theta\log\pi_\theta(a\mid s)\right\|^2
\,\middle|\, r(\theta),\widehat{A}
\right]=C
$
almost surely for some finite constant $C\geq 0$.
Assuming that the second moments below are finite, we have
\[
\mathbb{E}\!\left[
\left\|
r(\theta)\widehat{A}\nabla_\theta\log\pi_\theta(a\mid s)
\1_{\mathrm{ACPO}}
\right\|^2
\right]
\leq
\mathbb{E}\!\left[
\left\|
r(\theta)\widehat{A}\nabla_\theta\log\pi_\theta(a\mid s)
\1_{\mathrm{PPO}}
\right\|^2
\right].
\] 
\end{proposition}
\begin{proof}
    Note that { $(|r(\theta)\widehat{A}|^2-\alpha^2)(\1_\mathrm{PPO}-\1_\mathrm{ACPO})\geq 0$} for any $(s,a)$ by the definition of $\1_{\mathrm{ACPO}}$ and denote the gradient estimator as {$G_{\mathrm{PPO}}$}, {$G_\mathrm{ACPO}$}.
     \begin{align*}
        &\mathbb{E}\|G_{\mathrm{PPO}}\|^2- \mathbb{E}\|G_{\mathrm{ACPO}}\|^2\\
        =&\mathbb{E}[|r(\theta\widehat{A})|^2\|\nabla_{\theta}\log \pi_{\theta}(a\mid s)\|^2(\1_\mathrm{PPO}-\1_\mathrm{ACPO})]\\
        \geq& \mathbb{E}[\alpha^2\|\nabla_{\theta}\log \pi_{\theta}(a\mid s)\|^2(\1_\mathrm{PPO}-\1_\mathrm{ACPO})]\\
        =&
\alpha^2\mathbb E\left[
(\1_{\mathrm{PPO}}-\1_{\mathrm{ACPO}})
\mathbb E[\|\nabla_{\theta}\log \pi_{\theta}(a\mid s)\|^2\mid r(\theta),\widehat A]
\right]\\
=
&\alpha^2 C
\left(
\mathbb E[\1_{\mathrm{PPO}}]
-\mathbb E[\1_{\mathrm{ACPO}}]
\right)\\
=&0.
    \end{align*}
    This proves the claimed inequality.
\end{proof}

\section{Transformation of PMD with gradient clipping} \label{sec:transformation}
In this section we discuss how to transfer PMD with gradient clipping Algorithm~\ref{alg:pmd} to our ACPO objective \ref{eq: ACPO}.

First, we look at the PMD update \ref{alg: PMD} to solve minimizing total cost problem. The update takes following form:
\begin{align*}
    \pi_{k+1}(\cdot\mid s)
=
\argmin_{p\in\Delta(\mathcal{A})}
\left\{
\left\langle Q^{\pi_k}(s,\cdot),p+h^p(s)\right\rangle
+
\frac{1}{\eta_k}
D_{\psi}\left(p,\pi_k(\cdot\mid s)\right)
\right\}, \forall s\in \mathcal{S}.
\end{align*}

Here we replace $Q^{\pi_k}(s,\cdot)$ with $A^{\pi_k}(s,\cdot)$. Note this replacement does not the subproblem or its optimal solution. Then in the update, the term $A^{\pi_k}(s,\cdot)$ serves as gradient. Combining the gradient clipping technique, we have
\begin{align*}
    \pi_{k+1}(\cdot\mid s)
=
\argmin_{p\in\Delta(\mathcal{A})}
\left\{
\frac{\alpha}{\max \{\alpha, \|A^{\pi_k}(s,\cdot)\|\}}\left\langle A^{\pi_k}(s,\cdot),p+h^p(s)\right\rangle
+
\frac{1}{\eta_k}
D_{\psi}\left(p,\pi_k(\cdot\mid s)\right)
\right\}, \forall s\in \mathcal{S},
\end{align*}
which is equivalent to Algorithm~\ref{alg:pmd}.

Next, we take $h^p=0$ and $D_{\psi}\left(p,\pi_k(\cdot\mid s)\right)=\KL(p,{\pi_k})$. To match the objective of PPO, which maximizes the total reward, we consider the maximizing problem formulation. Then, we come to 
\begin{align*}
    \pi_{k+1}(\cdot\mid s)
=
\argmax_{p\in\Delta(\mathcal{A})}
\left\{
\frac{\alpha}{\max \{\alpha, \|A^{\pi_k}(s,\cdot)\|\}}\left\langle A^{\pi_k}(s,\cdot),p\right\rangle
-
\frac{1}{\eta_k}
\KL\left(p,\pi_k(\cdot\mid s)\right)
\right\}, \forall s\in \mathcal{S},
\end{align*}
which is equivalent to 
\[
\pi_{k+1}(\cdot \mid s)
\leftarrow
\arg\max_{\pi\in\Pi}
\mathbb{E}_{s \sim \rho_{\pi_k},a\sim \pi}
\left[
\clip (A^{\pi_k}(s,\cdot),\alpha)
\right]
-
\frac{1}{t_k}
\KL(\pi,\pi_k),
\]
if $\rho_{\pi_k}$ has full support. 

Here we move the clipping to each $A^{\pi_k}(s,a)$, because under a parameterized policy, the gradient contribution comes from each $A^{\pi_k}(s,a)$,  then obtain
\[
\pi_{k+1}(\cdot \mid s)
\leftarrow
\arg\max_{\theta}
\mathbb{E}_{(s,a) \sim \pi_{\theta_k}}
\left[
\clip (A^{\pi_{\theta_k}}(s,a),-\alpha,\alpha)
-
\frac{1}{t_k}
\KL(\pi,\pi_k)\right].
\]

Finally, we solve this subproblem by several optimization steps and then update the $\theta_k$. Because we need to use the estimated advantage from trajectory generated by old policy and integrate importance sampling, the objective is 
\[
\max_\theta \mathbb{E}_{(s,a)\sim \pi_{\mathrm{old}}}\left[\clip (r(\theta)\widehat{A}(s,a),-\alpha,\alpha)
-
\frac{1}{t_k}
\KL(\pi,\pi_\mathrm{old}) \right].
\]
After removing the $\KL$ term, it comes to our ACPO objective \ref{eq: ACPO}.

\section{Proof of Theorem \ref{thm: stongly convex} and \ref{thm: nonconvex} } \label{sec: convergence rate}
This section gives a convergence analysis for policy
mirror descent (PMD) with gradient clipping. Under relative strong convexity this objective-preserving
update converges linearly.  In the unregularized case, the same update reduces
to advantage clipping and enjoys a last-iterate $O(1/T)$ guarantee, including
for every fixed constant stepsize. 
\subsection{Notation}\label{sec:notation}
Let $\psi$ be a differentiable and strongly convex mirror map on the relative
interior of $\simplex$.  We use $D_\psi(x,y)$ to denote the Bregman divergence of $\psi$ between $x$ and $y$:
\begin{equation}\label{eq:bregman-generic}
 D_\psi(x,y):=\psi(x)-\psi(y)
 -\ip{\nabla\psi(y)}{x-y}.
\end{equation}
In particular, for
policies $\pi$ and $\pi'$, we use the following notation:
\begin{equation}\label{eq:bregman}
 D_{\pi'}^\pi(s):=
 D_\psi\bigl(\pi(\cdot\mid s),\pi'(\cdot\mid s)\bigr).
\end{equation}
We also write
$(\partial h)^{\pi'}(s,\cdot)$ for a subgradient of
$\pi\mapsto h^\pi(s)$ at $\pi'$.

Then we consider a $\gamma$-discounted MDP $\mathcal{M}=(\mathcal{S},\mathcal{A},P,r,\gamma,\rho)$, where $\mathcal{S}$ and $\mathcal{A}$ are the state and action spaces, $P(\cdot\mid s,a)$ is the transition kernel, $c(s,a)$ denotes the cost function (or $r(s,a)$ the reward function when maximizing reward), $\gamma\in[0,1)$ is the discount factor, and $\rho$ is the initial-state distribution. A stochastic policy $\pi$ specifies an action distribution $\pi(\cdot\mid s)$ at each state, and induces trajectories according to $s_0\sim\rho$, $a_t\sim\pi(\cdot\mid s_t)$, and $s_{t+1}\sim P(\cdot\mid s_t,a_t)$. 
Additionally, we consider a state-separable policy regularizer
$ h^\pi(s)$
, and we say that $h$ is
$\mu$-strongly convex relative to policy, if for every $\pi,\pi'$,
\begin{equation}\label{eq:relative-sc}
 h^\pi(s)\ge h^{\pi'}(s)
 +\ip{(\partial h)^{\pi'}(s,\cdot)}
 {\pi(\cdot\mid s)-\pi'(\cdot\mid s)}
 +\mu D_{\pi'}^\pi(s),
 \qquad \mu\ge0.
\end{equation}

The state-value and action-value functions of $\pi$ are defined as 
\begin{align} \label{eq:value}
 V^{\pi}(s)=\mathbb{E}_{\pi}[\sum_{t=0}^{\infty}\gamma^t (c(s_t,a_t)+h^\pi(s_t))\mid s_0=s],
 \end{align}
 and 
 \begin{align}\label{eq:q}
 Q^{\pi}(s,a)=\mathbb{E}_{\pi}[\sum_{t=0}^{\infty}\gamma^t (c(s_t,a_t)+h^\pi(s_t))\mid s_0=s,a_0=a]=c(s,a)+h^\pi(s)+\gamma\sum_{s'}P(s'\mid s,a)V^\pi(s').
 \end{align} 
 Then, advantage function is given by the difference of $Q^\pi(s,a)$ and $V^\pi(s)$:
 \begin{align} \label{advantage}
  A^{\pi}(s,a)=Q^{\pi}(s,a)-V^{\pi}(s).
 \end{align}

Consider an arbitrary norm $\norm{\cdot}$ on $\mathbb R^{|\Aset|}$ and a clipping
threshold $\alpha>0$.  At iteration $k$, set
\begin{equation}\label{eq:clipping}
 \lambda_k(s):=
 \frac{\alpha}{\max\{\alpha,\norm{A^{\pi_k}(s,\cdot)}\}}
\end{equation}
The stepsize $\eta_k>0$ is a single scalar independent of states.  The
objective-preserving clipped PMD update is
\begin{equation}\label{eq:update}
 \pi_{k+1}(\cdot\mid s)
 \in\arg\min_{p\in\simplex}
 \left\{
 \eta_k \lambda_k(s)\bigl[\ip{A^{\pi_k}(s,\cdot)}{p}+h_s(p)\bigr]
 +D_{\pi_k(\cdot\mid s)}^p(s)
 \right\}.
\end{equation}
Let $\pi^*$ be an optimal stationary policy, so
$V^{\pi^*}(s)\le V^\pi(s)$ for all $s$ and all stationary $\pi$.  Let $\nu^*$
be the distribution of the optimal policy, and define
\begin{align}
 e_k(s)&:=V^{\pi_k}(s)-V^{\pi^*}(s),
 &F_k&:=\E_{s\sim\nu^*}[e_k(s)],\label{eq:error}\\
 D_k(s)&:=D_{\pi_k(\cdot\mid s)}^{\pi^*(\cdot\mid s)}(s),
 &E_k(s)&:=D_{\pi_k(\cdot\mid s)}^{\pi_{k+1}(\cdot\mid s)}(s).
 \label{eq:divergences}
\end{align}
In particular, $e_k(s)\ge0$ and $F_k\ge0$.  For the KL divergence we take the
uniform initialization $\pi_0(a\mid s)=1/|\Aset|$, which gives
\begin{equation}\label{eq:initial-kl}
 D_0(s)=\KL\bigl(\pi^*(\cdot\mid s)\|\pi_0(\cdot\mid s)\bigr)
 \le\log|\Aset|.
\end{equation}

\subsection{Useful lemmas}\label{sec:lemmas}

For policies $\pi$ and $\pi'$, define the local improvement quantity
\begin{equation}\label{eq:g-general}
 g_{\pi,\pi'}(s):=
 \ip{A^\pi(s,\cdot)}{\pi'(\cdot\mid s)-\pi(\cdot\mid s)}
 +h^{\pi'}(s)-h^\pi(s).
\end{equation}
For the iterates, abbreviate $g_k:=g_{\pi_k,\pi_{k+1}}$ and we also define the comparison to the optimal policy
\begin{equation}\label{eq:b}
 b_k(s):=
 \ip{A^{\pi_k}(s,\cdot)}{\pi_k(\cdot\mid s)-\pi^*(\cdot\mid s)}
 +h^{\pi_k}(s)-h^{\pi^*}(s).
\end{equation}

\begin{lemma}[Performance difference]\label{lem:performance}
For any feasible policies $\pi$ and $\pi'$,
\begin{equation}\label{eq:performance-vector}
 V^{\pi'}-V^\pi=(I-\gamma P^{\pi'})^{-1}g_{\pi,\pi'}.
\end{equation}
Equivalently, we have
\begin{equation}\label{eq:performance-occupancy}
 V^{\pi'}(s)-V^\pi(s)
 =\frac{1}{1-\gamma}\E_{s'\sim d_s^{\pi'}}
   [g_{\pi,\pi'}(s')],
\end{equation}
where \begin{equation}\label{eq:occupancy}
 d_s^{\pi'}:=(1-\gamma)\sum_{t=0}^{\infty}\gamma^t
 e_s^\top(P^{\pi'})^t.
\end{equation}
\end{lemma}

\begin{proof}
  Without loss of generality, we assume that the policies are stationary.  
Fix a state $s$,  averaging the definition of $Q^\pi$ under
$\pi'(\cdot\mid s)$ and subtracting $V^\pi(s)$ gives
\begin{align*}
 \ip{A^\pi(s,\cdot)}{\pi'(\cdot\mid s)}
   +h^{\pi'}(s)-h^\pi(s)=\langle c(s,\cdot),\pi'(\cdot\mid s)\rangle+h^{\pi'}(s)
   +\gamma(P^{\pi'}V^\pi)(s)-V^\pi(s).
\end{align*}
Because $\ip{A^\pi(s,\cdot)}{\pi(\cdot\mid s)}=0$, the left-hand side is
exactly $g_{\pi,\pi'}(s)$.  The Bellman equation for $\pi'$ therefore implies
\begin{align*}
  V^\pi(s)
=
\sum_{a}\pi(a\mid s)
\left[
c(s,a)+h^\pi(s)
+\gamma\sum_{s'}P(s'\mid s,a)V^\pi(s')
\right],
\end{align*}
which is equivalent to
\begin{align*}
  V^{\pi'}(s)
=
\langle c(s,\cdot),\pi'(\cdot\mid s)\rangle+h^{\pi'}(s)
+\gamma(P^{\pi'}V^{\pi'})(s).
\end{align*}
Combining above results gives
\begin{align*}
 V^{\pi'}(s)-V^\pi(s)
 &=g_{\pi,\pi'}(s)
   +\gamma\bigl(P^{\pi'}(V^{\pi'}-V^\pi)\bigr)(s).
\end{align*}
Then we stack this identity over all states yields
\begin{equation*}
 (I-\gamma P^{\pi'})(V^{\pi'}-V^\pi)=g_{\pi,\pi'},
\end{equation*}
where since $P^{\pi'}$ is stochastic and $\gamma<1$, the Neumann series converges:
\begin{equation*}
 (I-\gamma P^{\pi'})^{-1}
 =\sum_{t=0}^{\infty}\gamma^t(P^{\pi'})^t.
\end{equation*}
This proves \eqref{eq:performance-vector}.  Left-multiplying by $e_s^\top$,
inserting the factor $(1-\gamma)/(1-\gamma)$, and using
\eqref{eq:occupancy} gives
\begin{align*}
 V^{\pi'}(s)-V^\pi(s)
 =\sum_{t=0}^{\infty}\gamma^t
   e_s^\top(P^{\pi'})^t g_{\pi,\pi'}=\frac{1}{1-\gamma}
   \E_{s'\sim d_s^{\pi'}}[g_{\pi,\pi'}(s')],
\end{align*}
giving us the desired result.
\end{proof}

\begin{lemma}[Optimal-policy identity]\label{lem:optimal-identity}
The vector $b_k$ satisfies
\begin{equation}\label{eq:b-identity}
 b_k=(I-\gamma P^{\pi^*})e_k,
\end{equation}
and consequently
\begin{equation}\label{eq:b-expectation}
 \E_{\nu^*}[b_k(s)]=(1-\gamma)F_k.
\end{equation}
\end{lemma}

\begin{proof}
For each state $s$, expand the inner product in \eqref{eq:b}:
\begin{align*}
 b_k(s)
 &=V^{\pi_k}(s)
   -\sum_a\pi^*(a\mid s)Q^{\pi_k}(s,a)
   +h^{\pi_k}(s)-h^{\pi^*}(s)\\
 &=V^{\pi_k}(s)-\langle c(s,\cdot),\pi^*(\cdot\mid s)\rangle-h^{\pi^*}(s)
   -\gamma(P^{\pi^*}V^{\pi_k})(s).
\end{align*}
The Bellman equation
$V^{\pi^*}(s)=\langle c(s,\cdot),\pi^*(\cdot\mid s)\rangle+h^{\pi^*}(s)+\gamma(P^{\pi^*}V^{\pi^*})(s)$ then gives
\begin{align*}
 b_k(s)
 &=V^{\pi_k}(s)-V^{\pi^*}(s)
   -\gamma\bigl(P^{\pi^*}(V^{\pi_k}-V^{\pi^*})\bigr)(s)\\
 &=\bigl((I-\gamma P^{\pi^*})e_k\bigr)(s),
\end{align*}
which proves \eqref{eq:b-identity}.  Finally,
\begin{align*}
 \E_{\nu^*}[b_k]
 &=(\nu^*)^\top(I-\gamma P^{\pi^*})e_k\\
 &=(1-\gamma)(\nu^*)^\top e_k=(1-\gamma)F_k,
\end{align*}
where stationarity gives
$(\nu^*)^\top P^{\pi^*}=(\nu^*)^\top$. 
\end{proof}

\begin{lemma}\label{lem:three-point}
For every $s\in\Sset$ and every feasible $p\in\simplex$, update
\eqref{eq:update} satisfies
\begin{align}
 &\eta_k\lambda_k(s)\Bigl[
 \ip{A^{\pi_k}(s,\cdot)}{\pi_{k+1}(\cdot\mid s)-p}
 +h^{\pi_{k+1}}(s)-h_s(p)\Bigr]+E_k(s)\le
 D_{\pi_k(\cdot\mid s)}^p(s)
 -\bigl(1+\mu\eta_k\lambda_k(s)\bigr)
 D_{\pi_{k+1}(\cdot\mid s)}^p(s).
 \label{eq:three-point}
\end{align}
\end{lemma}

\begin{proof}
Fix $s$ and take $x=\pi_k(\cdot\mid s)$,
$y=\pi_{k+1}(\cdot\mid s)$ and
$A=A^{\pi_k}(s,\cdot)$.  The objective minimized at this state is
\begin{equation*}
 p\longmapsto
 \eta_k\lambda_k(s)\bigl[\ip{A}{p}+h_s(p)\bigr]+D_x^p(s).
\end{equation*}
Since the feasible set $\simplex$ is convex, the first-order optimality
condition at $y$, with
$(\partial h)^{\pi_{k+1}}(s,\cdot)\in\partial h_s(y)$, gives
\begin{equation*}\label{eq:foc}
 \ip{\eta_k\lambda_k(s)
 \bigl(A+(\partial h)^{\pi_{k+1}}(s,\cdot)\bigr)
 +\nabla\psi(y)-\nabla\psi(x)}{p-y}\ge0.
\end{equation*}
Rearranging \eqref{eq:foc} yields
\begin{align*}
 &\eta_k\lambda_k(s)\ip{A}{y-p} \le
 \eta_k\lambda_k(s)
 \ip{(\partial h)^{\pi_{k+1}}(s,\cdot)}{p-y}
 +\ip{\nabla\psi(y)-\nabla\psi(x)}{p-y}.
\end{align*}
Since the regularizer $h^p$ is $\mu$-strongly convex w.r.t. policy, \eqref{eq:relative-sc} gives
\begin{equation*}
 \ip{(\partial h)^{\pi_{k+1}}(s,\cdot)}{p-y}
 \le h_s(p)-h_s(y)-\mu D_y^p(s).
\end{equation*}
Substituting this bound into the preceding inequality and moving the
regularizer difference to the left gives
\begin{align*}
 &\eta_k\lambda_k(s)
 \bigl[\ip{A}{y-p}+h_s(y)-h_s(p)\bigr]\le
 \ip{\nabla\psi(y)-\nabla\psi(x)}{p-y}
 -\mu\eta_k\lambda_k(s)D_y^p(s).
\end{align*}
Finally, direct expansion of the three Bregman divergences gives the
three-point identity
\begin{equation*}
 \ip{\nabla\psi(y)-\nabla\psi(x)}{p-y}
 =D_x^p(s)-D_x^y(s)-D_y^p(s)
\end{equation*}
and $D_x^y(s)=E_k(s)$.  Combining these identities proves \eqref{eq:three-point}.
\end{proof}

\begin{lemma}[Monotonic improvement]\label{lem:monotonicity}
For every state,
\begin{equation}\label{eq:g-bound}
 g_k(s)\le-\frac{E_k(s)}{\eta_k\lambda_k(s)}
 -\left(\frac{1}{\eta_k\lambda_k(s)}+\mu\right)
 D_{\pi_{k+1}}^{\pi_k}(s)
 \le0,
\end{equation}
and
\begin{equation}\label{eq:value-monotone}
 V^{\pi_{k+1}}(s)-V^{\pi_k}(s)\le g_k(s)\le0,
\end{equation}
leading to the conclusion that $V^{\pi_k}(s)$ and $F_k$ are nonincreasing in $k$.
\end{lemma}

\begin{proof}
Take $p=\pi_k(\cdot\mid s)$ in Equation~\ref{eq:three-point} and divide by
$\eta_k\lambda_k(s)>0$.  The bracket on the left becomes
\begin{align*}
 \ip{A^{\pi_k}(s,\cdot)}
 {\pi_{k+1}(\cdot\mid s)-\pi_k(\cdot\mid s)}
 +h^{\pi_{k+1}}(s)-h^{\pi_k}(s)=g_k(s),
\end{align*}
whereas the two divergences on the right are
$D_{\pi_k}^{\pi_k}(s)=0$ and
$D_{\pi_{k+1}}^{\pi_k}(s) $.  This gives the first inequality in
\eqref{eq:g-bound}.  Both remaining terms are nonpositive because Bregman
divergences are nonnegative, so $g_k(s)\le0$.

Apply Lemma~\ref{lem:performance} with
$(\pi,\pi')=(\pi_k,\pi_{k+1})$:
\begin{equation*}
 V^{\pi_{k+1}}(s)-V^{\pi_k}(s)
 =\frac{1}{1-\gamma}\sum_{s'}d_s^{\pi_{k+1}}(s')g_k(s').
\end{equation*}
The $t=0$ term in \eqref{eq:occupancy} shows that
$d_s^{\pi_{k+1}}(s)\ge1-\gamma$.  Since every $g_k(s')\le0$, discarding all
terms with $s'\ne s$ can only increase the sum.  Hence
\begin{align*}
 V^{\pi_{k+1}}(s)-V^{\pi_k}(s)
 \le\frac{d_s^{\pi_{k+1}}(s)}{1-\gamma}g_k(s)\leq g_k(s),
\end{align*}
where the second inequality uses both
$d_s^{\pi_{k+1}}(s)/(1-\gamma)\ge1$ and $g_k(s)\le0$.
This proves \eqref{eq:value-monotone}.  Subtracting the fixed value
$V^{\pi^*}$ and averaging under $\nu^*$ also shows that $F_k$ is
nonincreasing.
\end{proof}

\begin{lemma}[Fundamental recursion]\label{lem:fundamental}
The iterates of \eqref{eq:update} satisfy the following recursion for every $k$:
\begin{align}
 F_{k+1}
 +\E_{\nu^*}\!\left[
 \left(\frac{1}{\eta_k\lambda_k(s)}+\mu\right)D_{k+1}(s)\right]
 +\E_{\nu^*}\!\left[\frac{E_k(s)}{\eta_k\lambda_k(s)}\right] \le \gamma F_k+\E_{\nu^*}\!\left[
 \frac{D_k(s)}{\eta_k\lambda_k(s)}\right].
 \label{eq:fundamental}
\end{align}
\end{lemma}

\begin{proof}
Set $p=\pi^*(\cdot\mid s)$ in Equation~\ref{eq:three-point} and divide by
$\eta_k\lambda_k(s)$.  To identify the bracket, add and subtract
$\pi_k(\cdot\mid s)$ and $h^{\pi_k}(s)$:
\begin{align*}
 \ip{A^{\pi_k}}{\pi_{k+1}-\pi^*}
   +h^{\pi_{k+1}}-h^{\pi^*}&=\bigl(\ip{A^{\pi_k}}{\pi_{k+1}-\pi_k}
   +h^{\pi_{k+1}}-h^{\pi_k}\bigr)+\bigl(\ip{A^{\pi_k}}{\pi_k-\pi^*}
   +h^{\pi_k}-h^{\pi^*}\bigr)\\
 &
 =g_k(s)+b_k(s),
\end{align*}
where all policy arguments are evaluated at  any fixed state $s$.  Therefore
\begin{equation}\label{eq:before-average}
 b_k(s)+g_k(s)+\frac{E_k(s)}{\eta_k\lambda_k(s)}
 \le \frac{D_k(s)}{\eta_k\lambda_k(s)}
 -\left(\frac{1}{\eta_k\lambda_k(s)}+\mu\right)D_{k+1}(s).
\end{equation}
By Lemma~\ref{lem:monotonicity}, $V^{\pi_{k+1}}-V^{\pi_k}\le g_k$
statewise.  Thus replacing $g_k(s)$ on the left of
\eqref{eq:before-average} by
$V^{\pi_{k+1}}(s)-V^{\pi_k}(s)$ preserves the inequality.  Averaging with
respect to $\nu^*$ gives
\begin{align*}
 &\E_{\nu^*}[b_k]+F_{k+1}-F_k
 +\E_{\nu^*}\!\left[\frac{E_k(s)}{\eta_k\lambda_k(s)}\right]\le
 \E_{\nu^*}\!\left[\frac{D_k(s)}{\eta_k\lambda_k(s)}\right]
 -\E_{\nu^*}\!\left[
 \left(\frac{1}{\eta_k\lambda_k(s)}+\mu\right)D_{k+1}(s)\right].
\end{align*}
Now use $
 \E_{\nu^*}[b_k]=(1-\gamma)F_k$ in \eqref{eq:b-expectation},
so the first two scalar terms on the left combine as
$(1-\gamma)F_k+F_{k+1}-F_k=F_{k+1}-\gamma F_k$.
Reranging $-\gamma F_k$ and the $D_{k+1}$ term  yields
\eqref{eq:fundamental}.
\end{proof}

\begin{lemma}[Uniform lower bound on clipping]\label{lem:kappa}
Assume the cost function $c(s,a)$ and regularizer $h_s(p)$ are bounded. Then there exists a constant $\ell$ such that following definitions are well-defined and finite:
\begin{align*}
 K_{\mathrm{norm}}&:=\max_{\norm{x}_\infty\le1}\norm{x},\notag\\
 M_0&:=\max_sV^{\pi_0}(s),\notag\\
 C_{\mathrm{osc}}&:=\max_s\left(\max_a c(s,a)-\min_a c(s,a)\right),\notag\\
 B_A&:=K_{\mathrm{norm}}
 \left[C_{\mathrm{osc}}+\gamma
 \left(M_0-\frac{\ell}{1-\gamma}\right)\right].
\end{align*}
Then $\norm{A^{\pi_k}(s,\cdot)}\le B_A$ for every $k,s$, and hence
\begin{equation}\label{eq:kappa}
 \lambda_k(s)\ge\kappa:=\min\{1,\alpha/B_A\}>0,
\end{equation}
with the convention $\kappa=1$ if $B_A=0$.
\end{lemma}

\begin{proof}
For any policy $\pi$,
there exists a finite $\ell$ such that
\begin{equation}\label{eq:lower-stage}
 c(s,a)+h_s(p)\ge\ell
 \quad\text{for all $s,a$ and all $p\in\operatorname{dom}h_s$}.
\end{equation}
  Hence
\begin{equation*}
 V^\pi(s)\ge\sum_{t=0}^{\infty}\gamma^t\ell
 =\frac{\ell}{1-\gamma}.
\end{equation*}
Optimality of $\pi^*$ and Lemma~\ref{lem:monotonicity} therefore give, for every
$k$ and $s$,
\begin{equation*}
 \frac{\ell}{1-\gamma}\le V^{\pi^*}(s)\le V^{\pi_k}(s)\le M_0.
\end{equation*}
For any two actions $a,b$, the definition of $Q$ yields
\begin{align*}
 &\left|Q^{\pi_k}(s,a)-
 Q^{\pi_k}(s,b)\right|\le |c(s,a)-c(s,b)|
 +\gamma\left|
 \sum_{s'}\bigl(P(s'\mid s,a)-P(s'\mid s,b)\bigr)V^{\pi_k}(s')
 \right|.
\end{align*}
The two transition expectations both lie in the interval
$[\ell/(1-\gamma),M_0]$.  Their difference is therefore at most the length of
this interval, and consequently
\begin{equation*}
 \max_a Q^{\pi_k}(s,a)
 -\min_a Q^{\pi_k}(s,a)
 \le C_{\mathrm{osc}}+\gamma
 \left(M_0-\frac{\ell}{1-\gamma}\right).
\end{equation*}
Each advantage coordinate is obtained by
subtracting from $Q^{\pi_k}(s,a)$ a convex combination of the
coordinates of $Q^{\pi_k}(s,\cdot)$, because
$$A^\pi(s,a)=Q^\pi(s,a)
 -\sum_b\pi(b\mid s)Q^\pi(s,b).$$ Hence we have
\begin{align*}
 \norm{A^{\pi_k}(s,\cdot)}
 &\le K_{\mathrm{norm}}
 \norm{A^{\pi_k}(s,\cdot)}_\infty\le B_A.
\end{align*}
If $B_A>0$, the definition of \eqref{eq:clipping} gives
\begin{equation*}
 \lambda_k(s)
 \ge\min\!\left\{1,\frac{\alpha}{B_A}\right\}=\kappa.
\end{equation*}
If $B_A=0$, every advantage vector is zero, so we have
$\lambda_k(s)=1=\kappa$.
\end{proof}

The following lemma is needed only for a fixed stepsize in the unregularized
case.

\begin{lemma}[Finite relative variation of clipping weights]\label{lem:variation}
Suppose $h_s\equiv0$ and $\eta_k\equiv\eta>0$.  Let
\begin{equation}\label{eq:r-theta}
 r_k(s):=\frac1{\lambda_k(s)}
 =\max\!\left\{1,\frac{\norm{A^{\pi_k}(s,\cdot)}}{\alpha}\right\},
 \qquad
 \theta_k:=\max\!\left\{1,\max_s\frac{r_{k+1}(s)}{r_k(s)}\right\}.
\end{equation}
Then
\begin{equation}\label{eq:variation}
 \sum_{k=0}^{\infty}(\theta_k-1)
 \le T_0:=\frac{K_{\mathrm{norm}}}{\alpha}\sum_s e_0(s),
 \qquad
 R_N:=\prod_{k=0}^{N-1}\theta_k\le e^{T_0}.
\end{equation}
\end{lemma}

\begin{proof}
Set $u_k(s):=V^{\pi_k}(s)-V^{\pi_{k+1}}(s)\ge0$ and
$U_k:=\norm{u_k}_\infty$.  When $h\equiv0$,
\eqref{eq:q} and \eqref{advantage} give
\begin{equation*}
 A^\pi(s,a)=c(s,a)+\gamma\sum_{s'}P(s'\mid s,a)V^\pi(s')-V^\pi(s).
\end{equation*}
Subtracting this identity at $\pi_k$ from the one at $\pi_{k+1}$ proves
\begin{equation*}\label{eq:A-difference}
 A^{\pi_{k+1}}(s,a)-A^{\pi_k}(s,a)
 =u_k(s)-\gamma\sum_{s'}P(s'\mid s,a)u_k(s').
\end{equation*}
Because $0\le u_k(s')\le U_k$, the second term on the right belongs to
$[-\gamma U_k,0]$ and the first to $[0,U_k]$.  Thus the entire right-hand
side lies in $[-\gamma U_k,U_k]$, and in particular
\begin{equation*}
 \max_s\norm{A^{\pi_{k+1}}(s,\cdot)-A^{\pi_k}(s,\cdot)}
 \le K_{\mathrm{norm}}U_k.
\end{equation*}
The triangle inequality gives
\begin{equation*}
 \left|\norm{A^{\pi_{k+1}}(s,\cdot)}
 -\norm{A^{\pi_k}(s,\cdot)}\right|
 \le\norm{A^{\pi_{k+1}}(s,\cdot)-A^{\pi_k}(s,\cdot)}.
\end{equation*}
Since $x\mapsto\max\{1,x/\alpha\}$ is $1/\alpha$-Lipschitz,
\begin{equation*}
 |r_{k+1}(s)-r_k(s)|
 \le\frac{K_{\mathrm{norm}}}{\alpha}U_k.
\end{equation*}
Moreover, when $r_{k+1}(s)>r_k(s)$,
\begin{equation*}
 \frac{r_{k+1}(s)}{r_k(s)}-1
 =\frac{r_{k+1}(s)-r_k(s)}{r_k(s)}
 \le |r_{k+1}(s)-r_k(s)|.
\end{equation*}
Taking the maximum over states yields
\begin{equation*}
 \theta_k-1
 \le\max_s|r_{k+1}(s)-r_k(s)|
 \le\frac{K_{\mathrm{norm}}}{\alpha}U_k.
\end{equation*}
For every $N$,
\begin{align*}
 \sum_{k=0}^{N-1}U_k
 &\le\sum_s\sum_{k=0}^{N-1}u_k(s)
 =\sum_s\bigl(V^{\pi_0}(s)-V^{\pi_N}(s)\bigr)
 \le\sum_s e_0(s).
\end{align*}
Here the first inequality uses $U_k\le\sum_su_k(s)$, the equality telescopes,
and the last inequality uses $V^{\pi_N}(s)\ge V^{\pi^*}(s)$.  Combining the
last two displays and letting $N\to\infty$ proves the first claim.  Finally,
$\log x\le x-1$ for $x>0$, so
\begin{equation*}
 \log R_N=\sum_{k=0}^{N-1}\log\theta_k
 \le\sum_{k=0}^{N-1}(\theta_k-1)\le T_0.
\end{equation*}
\end{proof}

\subsection{Convergence theorem and proof}\label{sec:theorem}

\begin{theorem}[Convergence of clipped PMD, cf. Theorem~\ref{thm: stongly convex} and \ref{thm: nonconvex}]
\label{thm:main}
Assume the setting of Section~\ref{sec:notation}, use the KL divergence and uniform
initialization, and let $\kappa$ be the clipping lower bound in
Lemma~\ref{lem:kappa}.

\begin{enumerate}
\item[\rm(i)] \textbf{Strongly convex regularizer, fixed stepsize.}
Suppose $\mu>0$ and $\eta_k\equiv\eta>0$.  Define
\begin{equation*}\label{eq:rho-fixed}
 \rho:=\max\!\left\{\gamma,\frac{1}{\kappa(1+\eta\mu)}\right\}.
\end{equation*}
If $\eta\mu>1/\kappa-1$, then $\rho<1$ and, for every $N\ge0$, we have
\begin{equation*}\label{eq:linear-fixed}
 F_N+(\mu+\frac{1}{\eta})\E_{\nu^*}[D_N(s)]
 \le\rho^N\left(F_0+(\mu+\frac{1}{\eta})\E_{\nu^*}[D_0(s)]\right)
 \le\rho^N\left(F_0+(\mu+\frac{1}{\eta})\log|\Aset|\right).
\end{equation*}
In particular, if
$
 \eta\mu\ge\frac{1}{\gamma\kappa}-1,
$
then $\rho=\gamma$.

\item[\rm(ii)] \textbf{Strongly convex regularizer,  adaptive
stepsize.}  Suppose $\mu>0$ and choose
\begin{equation}\label{eq:adaptive-eta}
 \eta_{k+1}:=\eta_k
 \max\!\left\{1,\max_{s\in\Sset}\frac{\lambda_k(s)}{\lambda_{k+1}(s)}\right\}.
\end{equation}
Let $\kappa_0:=\min_s\lambda_0(s)$ and
$
 \rho_0:=\max\!\left\{\gamma,
 \frac{1}{1+\mu\eta_0\kappa_0}\right\}<1$.
Then
\begin{equation*}\label{eq:linear-adaptive}
 F_N\le\rho_0^N
 \left[F_0+\left(\mu+\frac{1}{\eta_0\kappa_0}\right)
 \log|\Aset|\right].
\end{equation*}

\item[\rm(iii)] \textbf{No regularizer,  adaptive stepsize.}
Suppose $h_s\equiv0$, and choose stepsizes as in \eqref{eq:adaptive-eta}.
Then, for every $N\ge1$,
\begin{equation*}\label{eq:sublinear-adaptive}
 F_N\le
 \frac{\gamma F_0+\E_{\nu^*}\!\left[
 D_0(s)/(\eta_0\lambda_0(s))\right]}{(1-\gamma)N}
 \le
 \frac{\gamma F_0+\dfrac{\log|\Aset|}{\eta_0\kappa_0}}
 {(1-\gamma)N}.
\end{equation*}

\item[\rm(iv)] \textbf{No regularizer, fixed stepsize.}
Suppose $h_s\equiv0$ and $\eta_k\equiv\eta>0$.  With $R_N$ and $T_0$ from
Lemma~\ref{lem:variation},
\begin{align*}
 F_N
 &\le\frac{R_N
 \left[\gamma F_0+\E_{\nu^*}\!\left[
 D_0(s)/(\eta\lambda_0(s))\right]\right]}
 {(1-\gamma)N}\\
 &\le\frac{
 \exp\!\left(\dfrac{K_{\mathrm{norm}}}{\alpha}\sum_se_0(s)\right)
 \left[\gamma F_0+\dfrac{\log|\Aset|}{\eta\kappa_0}\right]}
 {(1-\gamma)N}.
\end{align*}
\end{enumerate}
\end{theorem}

\begin{proof}
\textbf{Part (i).}
For a fixed stepsize and $\lambda_k(s)\in[\kappa,1]$,
\begin{equation*}
 \frac1\eta\le\frac{1}{\eta\lambda_k(s)}\le\frac{1}{\eta\kappa}.
\end{equation*}
Drop the nonnegative $E_k$ term in \eqref{eq:fundamental} and apply these
bounds.  On the left we use the lower bound $1/\eta$, while on the right we
use the upper bound $1/(\eta\kappa)$, obtaining
\begin{equation*}\label{eq:fixed-one-step}
 F_{k+1}+\left(\mu+\frac1\eta\right)\E_{\nu^*}[D_{k+1}]
 \le\gamma F_k+\frac{1}{\eta\kappa}\E_{\nu^*}[D_k].
\end{equation*}
By the choice of $\rho$, we  have
\begin{equation*}
 \gamma\le\rho,
 \qquad
 \frac{1}{\eta\kappa}\le\rho\left(\mu+\frac1\eta\right).
\end{equation*}
Hence the potential
$\Psi_k:=F_k+(\mu+\frac{1}{\eta})\E_{\nu^*}[D_k]$ has following property:
\begin{align*}
 \Psi_{k+1}
 \le\gamma F_k+\frac{1}{\eta\kappa}\E_{\nu^*}[D_k]\le\rho F_k+\rho(\mu+\frac{1}{\eta})\E_{\nu^*}[D_k]
 =\rho\Psi_k.
\end{align*}
Induction gives $\Psi_N\le\rho^N\Psi_0$.  Since the uniform initialization
satisfies $D_0(s)\le\log|\Aset|$ for every state, this is precisely
desired inequality.  Furthermore,
\begin{equation*}
 \frac{1}{\kappa(1+\eta\mu)}<1
 \quad\Longleftrightarrow\quad
 \eta\mu>\frac1\kappa-1,
\end{equation*}
so the stated condition makes both entries in the maximum defining $\rho$
strictly smaller than one. 

\medskip
\noindent\textbf{Part (ii).}
The stepsizes scheduled in \eqref{eq:adaptive-eta} implies
\begin{equation}\label{eq:effective-monotone}
 \eta_{k+1}\lambda_{k+1}(s)\ge\eta_k\lambda_k(s)
 \quad\text{for all $s$},
 \qquad
 \frac{1}{\eta_{k+1}\lambda_{k+1}(s)}
 \le\frac{1}{\eta_k\lambda_k(s)}.
\end{equation}
Indeed, the maximum in \eqref{eq:adaptive-eta} is at least
$\lambda_k(s)/\lambda_{k+1}(s)$ for each fixed $s$; multiplication by
$\eta_k\lambda_{k+1}(s)$ proves the first inequality, and taking reciprocals
proves the second.
Define the time-varying potential
\begin{equation*}
 \Phi_k:=F_k+\E_{\nu^*}\!\left[
 \left(\frac{1}{\eta_k\lambda_k(s)}+\mu\right)D_k(s)\right].
\end{equation*}
Using \eqref{eq:effective-monotone} and then \eqref{eq:fundamental},
\begin{align*}
 \Phi_{k+1}
 &\le F_{k+1}+\E_{\nu^*}\!\left[
 \left(\frac{1}{\eta_k\lambda_k(s)}+\mu\right)D_{k+1}(s)\right] \leq \gamma F_k+\E_{\nu^*}\!\left[
 \frac{D_k(s)}{\eta_k\lambda_k(s)}\right].
\end{align*}
Moreover, $\eta_k\lambda_k(s)\ge\eta_0\lambda_0(s)
\ge\eta_0\kappa_0$, and therefore
\begin{equation*}
 \frac{1/(\eta_k\lambda_k(s))}
 {1/(\eta_k\lambda_k(s))+\mu}
 =\frac{1}{1+\mu\eta_k\lambda_k(s)}
 \le\frac{1}{1+\mu\eta_0\kappa_0}\le\rho_0.
\end{equation*}
Equivalently, for every state,
\begin{equation*}
 \frac{1}{\eta_k\lambda_k(s)}
 \le\rho_0\left(\frac{1}{\eta_k\lambda_k(s)}+\mu\right).
\end{equation*}
Together with $\gamma\le\rho_0$, this bounds the last line in the display for
$\Phi_{k+1}$ by $\rho_0\Phi_k$.  Hence
$\Phi_N\le\rho_0^N\Phi_0$.  Dropping the nonnegative divergence term from
$\Phi_N$ and using
\begin{align*}
 \Phi_0
 \le F_0+\left(\mu+\frac{1}{\eta_0\kappa_0}\right)
 \E_{\nu^*}[D_0(s)] \le F_0+\left(\mu+\frac{1}{\eta_0\kappa_0}\right)
 \log|\Aset|,
\end{align*}
we have the desired result.

\medskip
\noindent\textbf{Part (iii).}
With $\mu=0$, define
\begin{equation*}
 W_k:=\E_{\nu^*}\!\left[\frac{D_k(s)}{\eta_k\lambda_k(s)}\right],
 \qquad Y_k:=\gamma F_k+W_k.
\end{equation*}
Since \eqref{eq:effective-monotone} makes the reciprocal coefficients
nonincreasing, \eqref{eq:fundamental} gives
\begin{equation*}
 F_{k+1}+W_{k+1}\le\gamma F_k+W_k=Y_k.
\end{equation*}
Indeed, the $D_{k+1}$ coefficient defining $W_{k+1}$ is no larger than the
corresponding coefficient on the left of \eqref{eq:fundamental}; the omitted
$E_k$ term is nonnegative.
Equivalently,
\begin{equation}\label{eq:Y-telescope}
 Y_{k+1}+(1-\gamma)F_{k+1}\le Y_k.
\end{equation}
Summing \eqref{eq:Y-telescope} from $k=0$ to $N-1$ telescopes the $Y_k$ terms:
\begin{equation*}
 Y_N+(1-\gamma)\sum_{j=1}^NF_j\le Y_0.
\end{equation*}
Since $Y_N\ge0$ and $F_j\ge F_N$ for $1\le j\le N$ by
Lemma~\ref{lem:monotonicity},
\begin{equation*}
 (1-\gamma)NF_N
 \le(1-\gamma)\sum_{j=1}^NF_j
 \le Y_0.
\end{equation*}
This proves the first inequality in \eqref{eq:sublinear-adaptive}.  The second
follows from
\begin{align*}
 \E_{\nu^*}\!\left[\frac{D_0(s)}{\eta_0\lambda_0(s)}\right]
 &\le\frac{1}{\eta_0\kappa_0}\E_{\nu^*}[D_0(s)]
 \le\frac{\log|\Aset|}{\eta_0\kappa_0}.
\end{align*}

\medskip
\noindent\textbf{Part (iv).}
For a fixed stepsize, Lemma~\ref{lem:variation} gives
$1/(\eta\lambda_{k+1}(s))
\le\theta_k/(\eta\lambda_k(s))$, because
$r_{k+1}(s)\le\theta_kr_k(s)$.  From \eqref{eq:fundamental} with $\mu=0$,
\begin{equation*}
 \E_{\nu^*}\!\left[\frac{D_{k+1}(s)}{\eta\lambda_k(s)}\right]
 \le Y_k-F_{k+1}.
\end{equation*}
Consequently,
\begin{align}
 Y_{k+1}=\gamma F_{k+1}+W_{k+1}\notag\le\gamma F_{k+1}
 +\theta_k\E_{\nu^*}\!\left[
 \frac{D_{k+1}(s)}{\eta\lambda_k(s)}\right] \le\theta_kY_k-(\theta_k-\gamma)F_{k+1}.
 \label{eq:Y-weighted}
\end{align}
The first inequality uses the preceding comparison of reciprocal clipping
factors; the second substitutes the bound from \eqref{eq:fundamental}.
Set $R_0=1$ and $R_{k+1}=\theta_kR_k$.  Dividing
above inequality by $R_{k+1}$ gives
\begin{equation*}
 \frac{Y_{k+1}}{R_{k+1}}
 \le\frac{Y_k}{R_k}
 -\frac{(\theta_k-\gamma)F_{k+1}}{R_{k+1}}.
\end{equation*}
Summing from $k=0$ to $N-1$ and using
$\theta_k-\gamma\ge1-\gamma$ gives
\begin{equation*}
 (1-\gamma)\sum_{j=1}^N\frac{F_j}{R_j}
 \le Y_0-\frac{Y_N}{R_N}\le Y_0,
\end{equation*}
where the last inequality follows from $Y_N\ge0$.  Since $F_j\ge F_N$ by
monotonicity and $R_j\le R_N$ because each $\theta_k\ge1$,
\begin{equation*}
 \sum_{j=1}^N\frac{F_j}{R_j}\ge\frac{NF_N}{R_N}.
\end{equation*}
Combining the last two displays yields
\begin{equation*}
 F_N\le\frac{R_NY_0}{(1-\gamma)N}.
\end{equation*}  
Finally, using $R_N\le e^{T_0}$ in Lemma~\ref{lem:variation}, \eqref{eq:initial-kl}  and the definition of
$T_0$, we obtain

$$
 \E_{\nu^*}\!\left[\frac{D_0(s)}{\eta\lambda_0(s)}\right]
 \le\frac{\log|\Aset|}{\eta\kappa_0},
$$

\end{proof}

\begin{remark}[From $\nu^*$-weighted error to all-state error]
If $\nu_{\min}^*:=\min_s\nu^*(s)>0$, then every bound on $F_N$ immediately
implies
\begin{equation*}
 \norm{V^{\pi_N}-V^{\pi^*}}_\infty
 \le\frac{F_N}{\nu_{\min}^*}.
\end{equation*}
Without full support, one may instead choose any full-support distribution
$\rho$ and set
\begin{equation*}
 \widetilde\nu^\top:=(1-\gamma)\rho^\top(I-\gamma P^{\pi^*})^{-1}.
\end{equation*}
Then $\widetilde\nu(s)>0$ and we have 
\begin{equation*}
 \E_{\widetilde\nu}[b_k]=(1-\gamma)\E_\rho[e_k]
 \ge\delta\E_{\widetilde\nu}[e_k].
\end{equation*}
Now we define $\delta:=(1-\gamma)\min_s\frac{\rho(s)}{\widetilde\nu(s)}>0$,
then repeat the proof with $\widetilde\nu$. This replaces the contraction coefficient
$\gamma$ in the fundamental recursion Lemma~\ref{lem:fundamental} by $1-\delta<1$ and yields similar convergence results on all states.
\end{remark}

\section{Additional Experimental  Details} \label{Additional Experimental  Details}
\subsection{Training Details} \label{training details}
All experiments are conducted on a single node with parameter and optimizer offloading enabled. We use eight NVIDIA H100 GPUs for Qwen2.5-Math-7B and Qwen3-8B, four for Qwen3-4B, and two for Qwen3-1.7B. 

We post-train models with our algorithm using VERL framework \citep{Sheng2025verl}. And all the hyperparameters are presented in Table~\ref{tab:rl_hyperparameters}.

\begin{table}[ht] 
\centering
\begin{threeparttable}
\caption{RL training hyperparameters.
Qwen3 includes the 1.7B, 4B, and 8B models.
Ratio clipping thresholds apply to the baselines,
whereas $\alpha$ applies to ACPO variants.
All sampling parameters refer to training rollouts.}
\label{tab:rl_hyperparameters}
\small
\setlength{\tabcolsep}{4pt}
\begin{tabular}{lcccc}
\toprule
& \multicolumn{2}{c}{PPO-type}
& \multicolumn{2}{c}{GRPO-type} \\
\cmidrule(lr){2-3} \cmidrule(lr){4-5}
Hyperparameter
& Qwen2.5-Math-7B\tnote{a} & Qwen3
& Qwen2.5-Math-7B\tnote{a} & Qwen3 \\
\midrule
\multicolumn{5}{l}{\textit{Data and rollout}} \\
Training dataset          & DAPO & DAPO & DAPO & DAPO \\
Prompt batch size         & 512 & 512 & 512 & 256 \\
Rollouts per prompt       & 1 & 1 & 8 & 4 \\
Maximum prompt length     & 1024 & 1024 & 1024 & 1024 \\
Maximum response length   & 8196 & 8196 & 8196 & 8196 \\
Sampling temperature      & 1.0 & 1.0 & 1.0 & 1.0 \\
Top-$p$                   & 1.0 & 1.0 & 1.0 & 1.0 \\
Top-$k$                   & -- & -- & -- & -- \\
\midrule
\multicolumn{5}{l}{\textit{Policy optimization}} \\
Optimizer  & AdamW & AdamW & AdamW & AdamW \\
Actor learning rate
& $10^{-6}$ & $10^{-6}$\tnote{b}
& $10^{-6}$ & $10^{-6}$ \\
Learning-rate schedule    & Constant & Constant & Constant & Constant \\
Actor mini-batch size     & 128 & 128 & 128 & 64 \\
Gradient clipping norm    & 1.0 & 1.0 & 1.0 & 1.0 \\
Entropy coefficient       & 0 & 0 & 0 & 0 \\
KL penalty in reward      & No & No & No & No \\
KL loss coefficient       & 0 & 0 & 0 & 0 \\
Baseline $\epsilon_{\mathrm{low}}$
                          & 0.2 & 0.2 & 0.2 & 0.2 \\
Baseline $\epsilon_{\mathrm{high}}$
                          & 0.2 & 0.2 & 0.28 & 0.28 \\
ACPO threshold $\alpha$   & 3 & 3 & 3 & 2 \\
Critic learning rate      & $10^{-5}$ & $10^{-5}$ & -- & -- \\
Total training steps      & 300 & 300 & 200 & 100 \\
Total optimizer steps
                          & 1200 & 1200 & 800 & 400 \\
\bottomrule
\end{tabular}

\begin{tablenotes}[flushleft]
\footnotesize
\item[a] The maximum context length of
Qwen2.5-Math-7B is 4096 tokens.
\item[b] The actor learning rate is $5\times10^{-7}$
for Qwen3-1.7B in PPO-type experiments.

\end{tablenotes}
\end{threeparttable}
\end{table}

\subsection{Detailed Benchmark Results}
In this section, we report results on four benchmarks: \textbf{MATH500} \citep{hendrycks2021math500}, \textbf{Minerva Math} \citep{Lewkowycz2022Minerva}, \textbf{OlympiadBench} \citep{he2024olympiadbench}, and \textbf{AIME-like} \citep{xiong2025reinforceadaadaptivesamplingframework}. AIME-like comprises 230 problems from recent competitions: AIME24, AIME25, HMMT24, HMMT25, BRUMO25, AMC23, and CMIMC25. We estimate pass@1 accuracy by averaging correctness over 32 sampled responses per problem for Qwen2.5-7B-Math (Avg@32) and 16 for Qwen3 (Avg@16). All responses are generated with a temperature of 1.0, top-$p=1$, and a maximum generation length of 8,196 tokens. 

As shown in Figures~\ref{fig:2.5-7}, \ref{fig:3-1.7},\ref{fig:3-4} and \ref{fig:3-8}, ACPO consistently outperforms the PPO and GRPO baselines, particularly on challenging benchmarks such as \textbf{OlympiadBench} \citep{he2024olympiadbench} and \textbf{AIME-like} \citep{xiong2025reinforceadaadaptivesamplingframework}. Beyond improving performance, ACPO substantially accelerates training, with particularly pronounced gains for PPO-AC.

\begin{figure}[ht]
    \centering
    \includegraphics[width=0.99\linewidth]{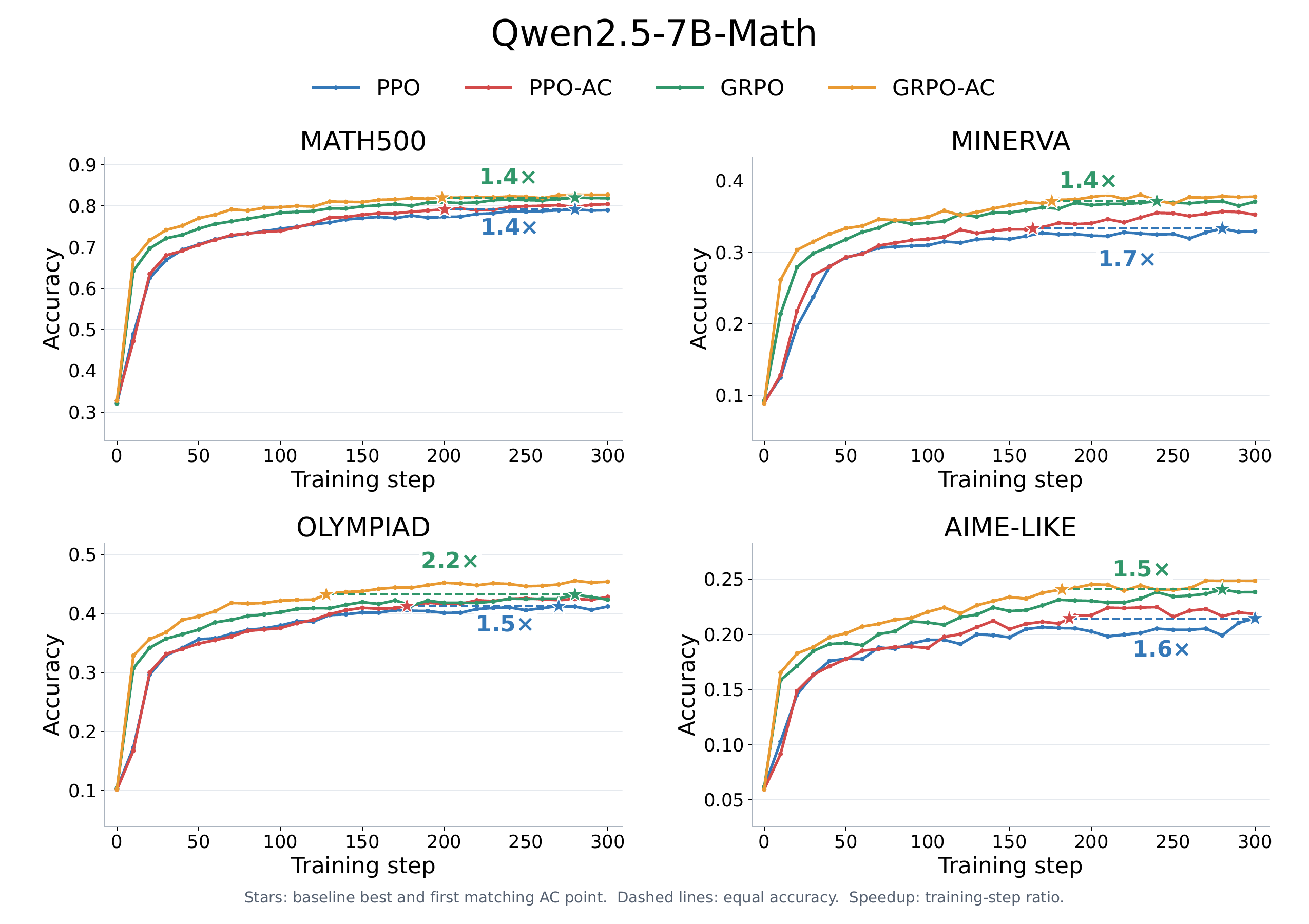}
    \caption{Qwen2.5-7B-Math on benchmarks}
    \label{fig:2.5-7}
\end{figure}

\begin{figure}[ht]
    \centering
    \includegraphics[width=0.99\linewidth]{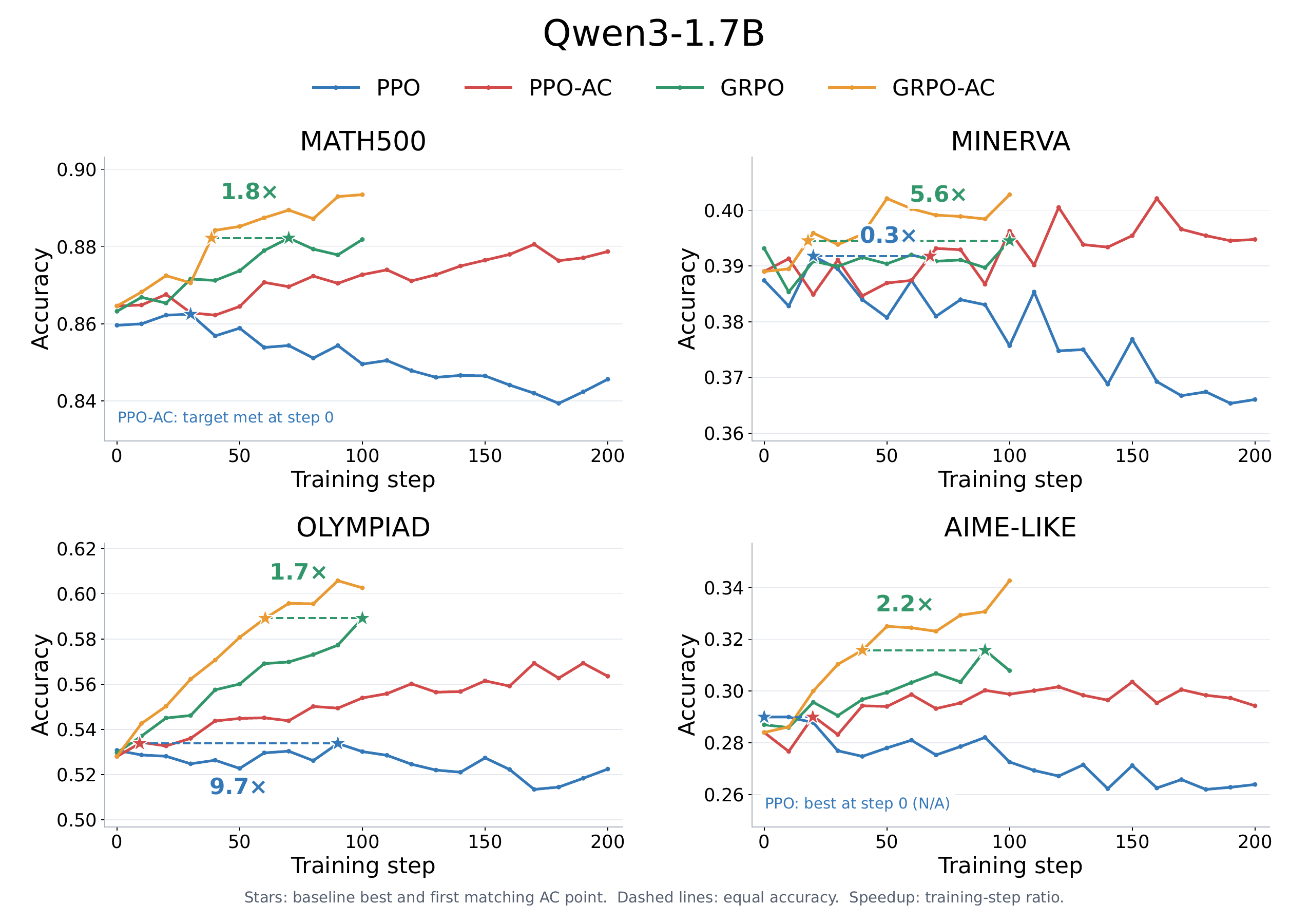}
    \caption{Qwen3-1.7B on benchmarks}
    \label{fig:3-1.7}
\end{figure}

\begin{figure}[ht]
    \centering
    \includegraphics[width=0.99\linewidth]{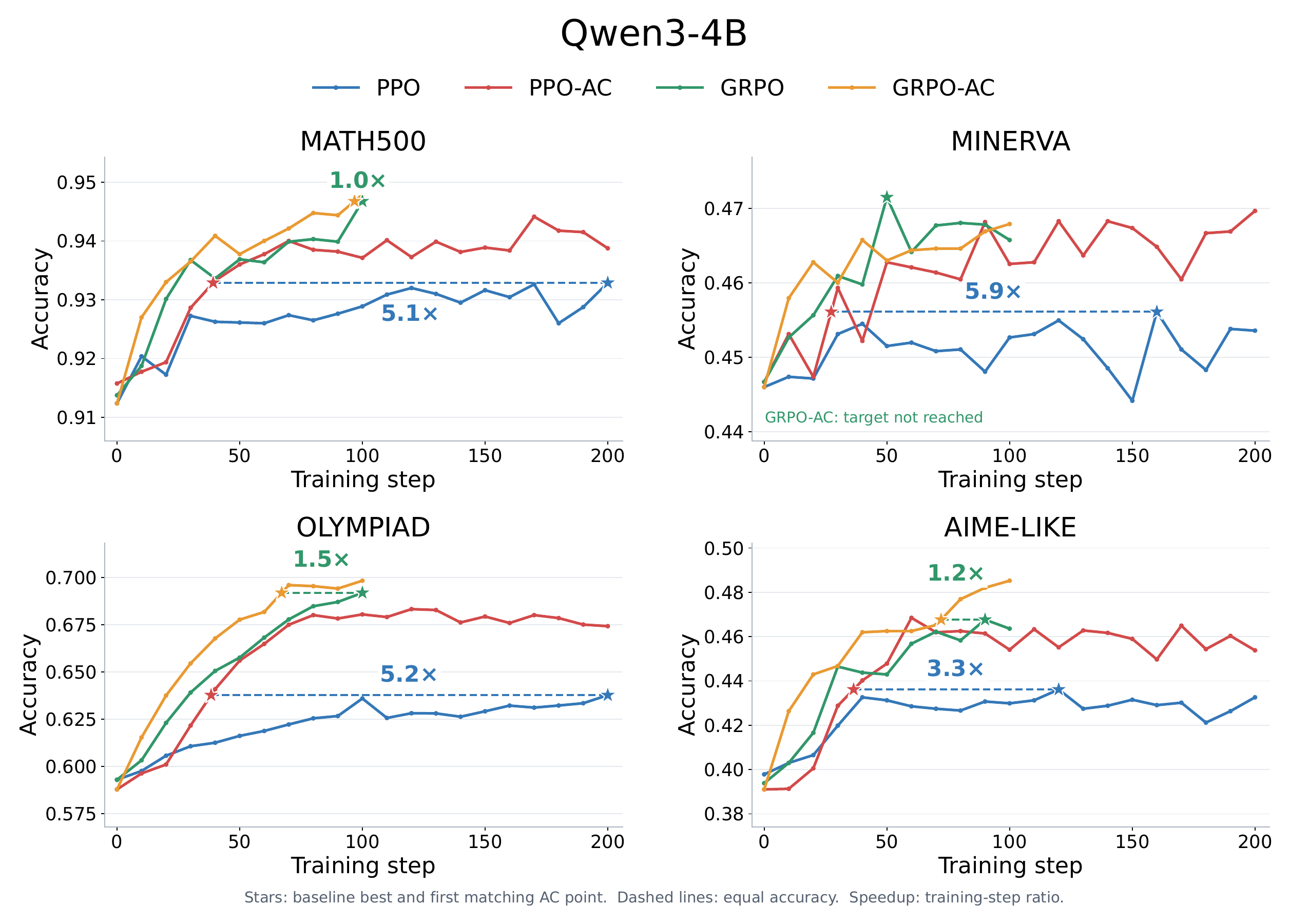}
    \caption{Qwen3-4B on  benchmarks}
    \label{fig:3-4}
\end{figure}

\begin{figure}[ht]
    \centering
    \includegraphics[width=0.99\linewidth]{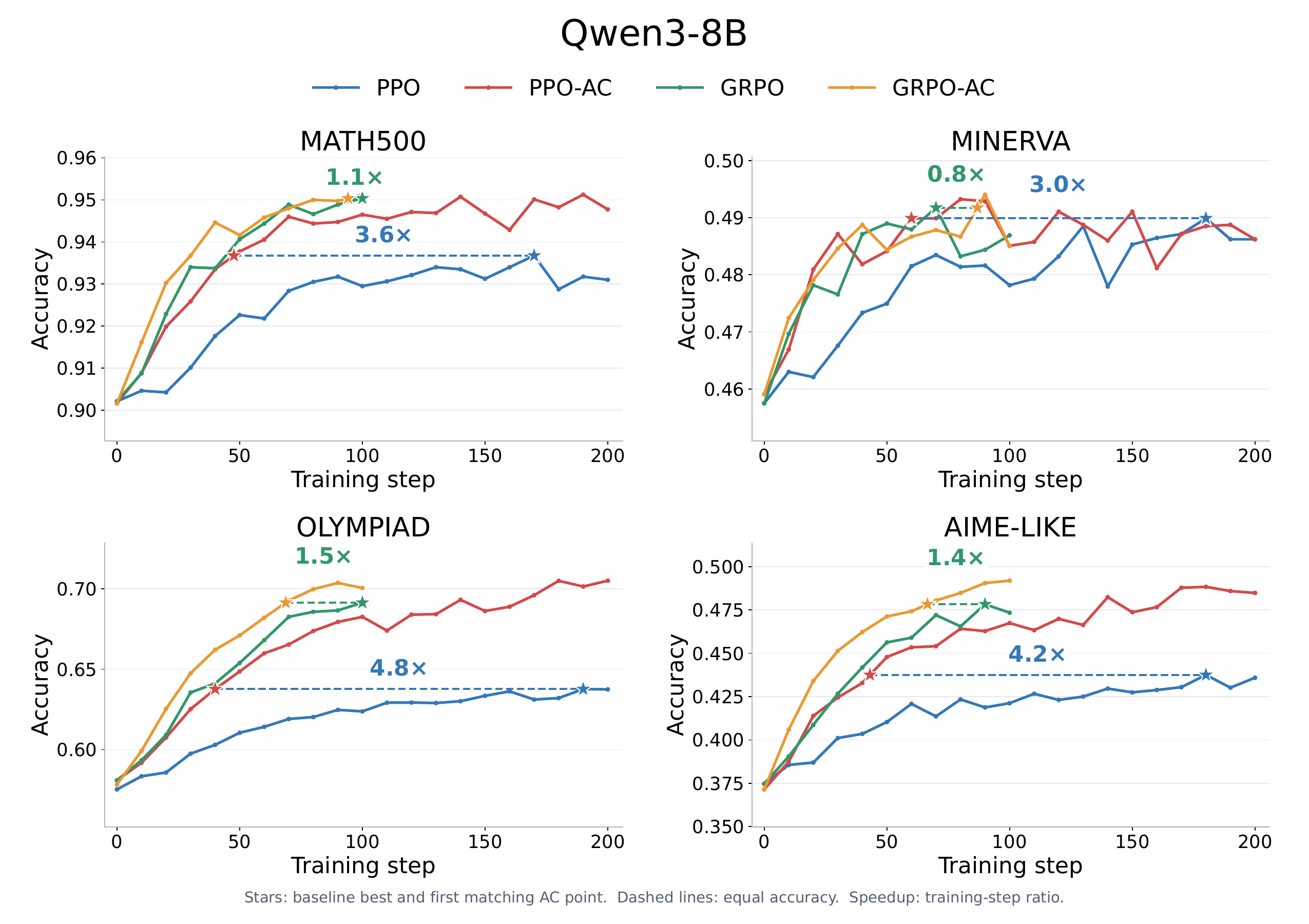}
    \caption{Qwen3-8B on benchmarks}
    \label{fig:3-8}
\end{figure}

\subsection{Additional Results for Qwen3-1.7B with PPO-Type Learning Rate $10^{-6}$}
Here we report the post-training result of Qwen3-1.7B with lr=1e-6 for PPO-type algorithms (Figure~\ref{fig:1e-6} and \ref{fig:1e-62}). On Qwen3-1.7B, with the learning rate set to $10^{-6}$ for both PPO and PPO-AC, PPO-AC achieves a peak weighted accuracy of 60.55\% at step 110, outperforming PPO’s best accuracy of 57.37\% by 3.18 percentage points within  200 steps training. Notably, PPO attains its best performance at initialization, whereas PPO-AC improves beyond this baseline, demonstrating more effective policy optimization in this setting. Hence we tune the learning rate such that PPO do not keep decrease during the training. We search for 5 learning rates, and finally choose lr=5e-7 for Qwen3-1.7B. Overall, the results on Qwen3-1.7B with a learning rate of $10^{-6}$ demonstrate the substantial performance gains achieved by ACPO over PPO.

\begin{figure}[ht]
    \centering
    \includegraphics[width=0.99\linewidth]{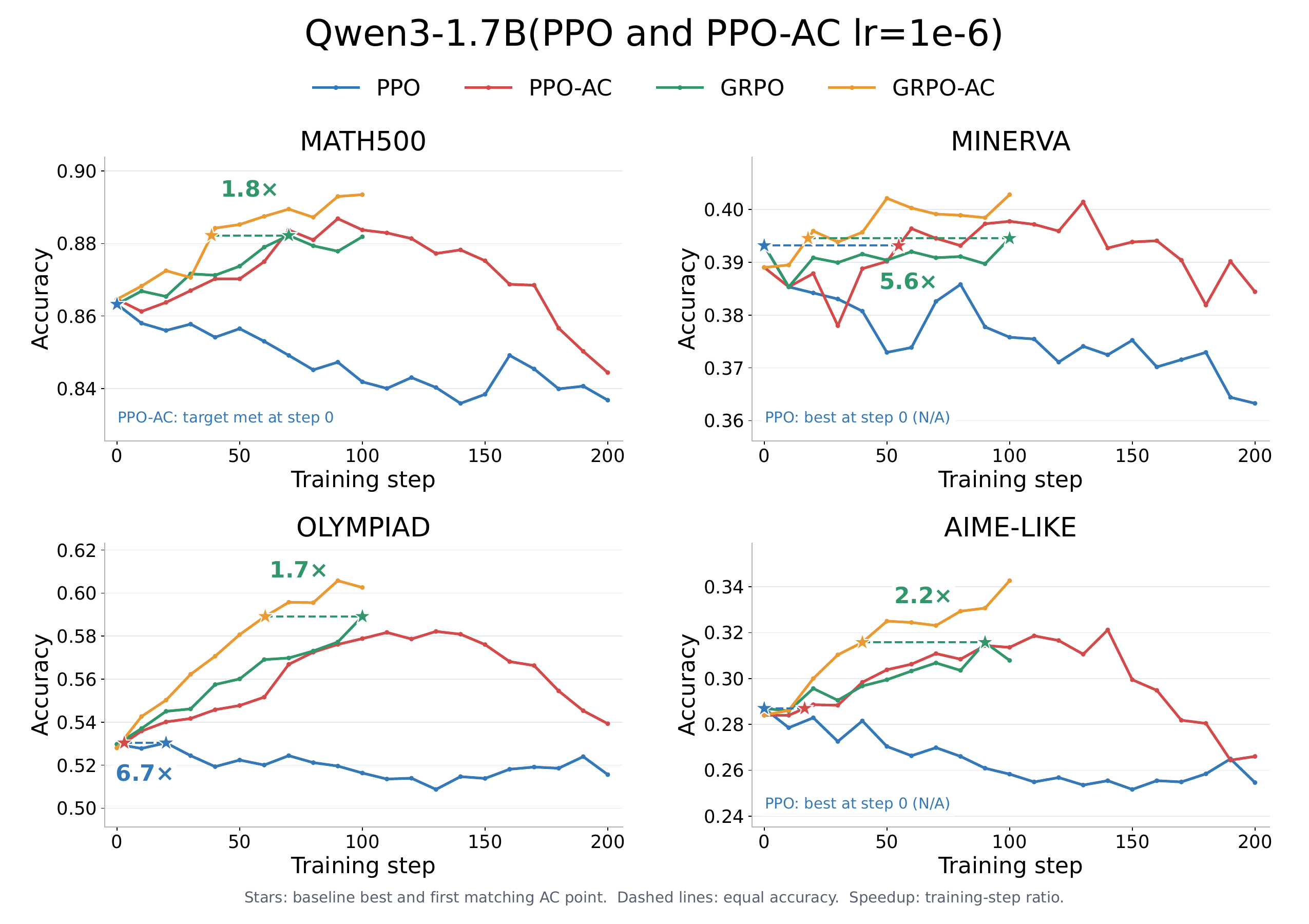}
    \caption{Qwen3-1.7B with PPO-Type lr=1e-6 on  benchmarks}
    \label{fig:1e-6}
\end{figure}

\begin{figure}[ht]
    \centering
    \includegraphics[width=0.99\linewidth]{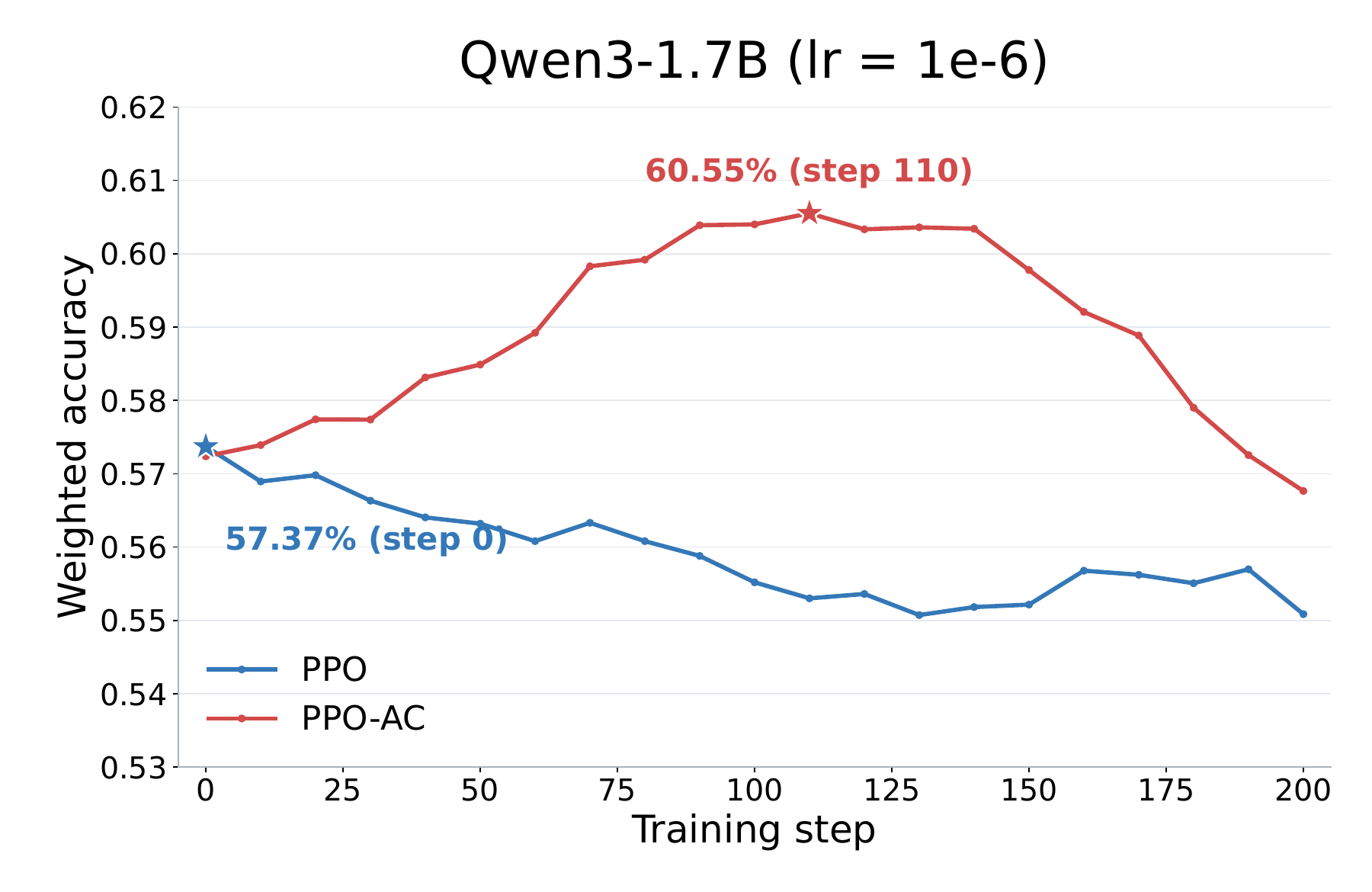}
    \caption{Weighted Accuracy of Qwen3-1.7B with PPO-Type lr=1e-6}
    \label{fig:1e-62}
\end{figure}

\end{document}